\documentclass[a4paper]{article}
\usepackage[utf8]{inputenc}
\usepackage{authblk}
\usepackage{amssymb}
\usepackage[sort&compress,numbers,square]{natbib}
\usepackage[titlenumbered,ruled,lined,linesnumbered,algo2e]{algorithm2e}
\usepackage{graphicx}
\usepackage{amsmath}
\usepackage{xspace}
\usepackage{bm}
\usepackage{amsthm}
\usepackage{afterpage}
\usepackage{indentfirst}
\usepackage{enumerate}
\usepackage{url}
\usepackage{comment}
\usepackage[scale={0.72,0.743}, hmarginratio=1:1, vmarginratio=1:1, headheight=2ex, headsep = 2ex, footskip = 3.6ex, nomarginpar]{geometry}
\usepackage{color}
\usepackage{mathtools}
\usepackage{hyperref}

\newtheorem{theorem}{Theorem}
\newtheorem{lemma}[theorem]{Lemma}
\newtheorem{proposition}[theorem]{Proposition}
\newtheorem{corollary}[theorem]{Corollary}
\theoremstyle{definition}

\newtheorem{assumption}{Assumption}

\theoremstyle{definition}

\newtheorem{remark}{Remark}

\newcommand{\nbb}{\ensuremath{\mathbb{N}}}

\newcommand{\fcal}{\mathcal{F}}
\newcommand{\xcal}{\mathcal{X}}
\newcommand{\zcal}{\mathcal{Z}}
\newcommand{\ecal}{\mathcal{E}}
\newcommand{\lcal}{\mathcal{L}}

\newcommand{\ycal}{\mathcal{Y}}
\newcommand{\gcal}{\mathcal{G}}
\newcommand{\ebb}{\mathbb{E}}
\newcommand{\rbb}{\mathbb{R}}
\newcommand{\inn}[1]{\langle#1\rangle}
\newcommand{\half}{\lfloor\frac{T}{2}\rfloor}
\numberwithin{equation}{section}

\title{Convergence of Unregularized Online Learning Algorithms}
\author[1]{Yunwen Lei\thanks{yunwelei@cityu.edu.hk}}
\author[2]{Lei Shi\thanks{leishi@fudan.edu.cn}}
\author[3]{Zheng-Chu Guo\thanks{guozhengchu@zju.edu.cn}}
\affil[1]{Department of Mathematics, City University of Hong Kong, Hong Kong, China}
\affil[2]{School of Mathematical Sciences, Shanghai Key Laboratory for Contemporary Applied Mathematics,
       Fudan University, Shanghai 200433, China}
\affil[3]{School of Mathematical Sciences,
       Zhejiang University, Hangzhou 310027, China}

\begin{document}
\maketitle
\linespread{1.1}

\begin{abstract}
In this paper we study the convergence of online gradient descent algorithms in reproducing kernel Hilbert spaces (RKHSs) without regularization. We establish a sufficient condition and a necessary condition for the convergence of excess generalization errors in expectation. A sufficient condition for the almost sure convergence is also given. With high probability, we provide explicit convergence rates of the excess generalization errors for both averaged iterates and the last iterate, which in turn also imply convergence rates with probability one.
To our best knowledge, this is the first high-probability convergence rate for the last iterate of online gradient descent algorithms without strong convexity.
Without any boundedness assumptions on iterates, our results are derived by a novel use of two measures of the algorithm's one-step progress, respectively by generalization errors and by distances in RKHSs, where the variances of the involved martingales are cancelled out by the descent property of the algorithm.

\end{abstract}

\section{Introduction}
Online gradient descent is a scalable method able to tackle large-scale data arriving in a sequential manner~\citep{zhang2004solving,kivinen2004online,duchi2009efficient,dieuleveut2016nonparametric}, which is becoming ubiquitous within the big data era. As a first-order method, it iteratively builds an unbiased estimate of the true gradient upon the arrival of a new example and uses this information to guide the learning process~\citep{zinkevich2003online,zhang2004solving}.
As verified by theoretical and empirical analysis, online gradient descent enjoys comparable performance as compared to its batch counterpart such as gradient descent \citep{zhang2004solving,yao2010complexity,shalev2011pegasos}, while attaining a great computational speed-up since its gradient calculation involves only a single example. As a comparison, the gradient calculation in gradient descent requires to traverse all training examples. Recently, online gradient descent has received renewed attention due to the wide applications of its stochastic analogue, i.e., stochastic gradient descent, in training deep neural networks~\citep{bottou1991stochastic,ngiam2011optimization,sutskever2013importance}.

In this paper, we are interested in the setting that training examples $\{z_t=(x_t,y_t)\}_{t\in\nbb}$ are sequentially and identically drawn from a probability measure $\rho$ defined in the sample space $\zcal=\xcal\times\ycal$, where $\xcal\subset\rbb^d$ is the input space and $\ycal\subset\rbb$ is the output space.
We focus on the nonparametric setting, where the learning process is implemented in a reproducing kernel Hilbert space (RKHS) $H_K$ associated with a Mercer kernel $K:\xcal\times\xcal\to\rbb$ which is assumed to be continuous, symmetric and positive semi-definite. The space $H_K$ is defined as the completion of the linear span of the set of functions $\{K_x(\cdot):=K(x,\cdot):x\in\xcal\}$ with the inner producing satisfying the reproducing property $f(x)=\inn{f,K_x}$ for any $x\in\xcal$ and $f\in H_K$. In this setting, the use of Mercer kernels provides a unifying way to measure similarities between pairs of objects~\citep{cortes1995support,muller2001introduction,steinwart2001influence,scholkopf2001learning}, which turns out to be a key to the great success of kernel methods in many practical learning problems.
We wish to build a prediction rule $f\in H_K$ after seeing a sequence of training examples, the performance of which at an example $(x,y)$ can be quantitatively measured by a loss function $\phi:\ycal\times\rbb\to\rbb_+$ as $\phi(y,f(x))$.
With a sequence $\{\eta_t\}_{t\in\nbb}$ of positive step sizes and $f_1=0$, online gradient descent is a realization of learning schemes by keeping a sequence of iterates as follows
\begin{equation}\label{OKL}
  f_{t+1}=f_t-\eta_t\phi'(y_t,f_t(x_t))K_{x_t},\quad\forall t\in\nbb,
\end{equation}
where $\phi'$ denotes the derivative of $\phi$ with respect to the second argument. Although our focus is on the nonparametric setting, it should be mentioned that the above algorithm also recovers the parametric case in which the kernel is taken to be the linear kernel with $K_x(x')=\langle x,x' \rangle, \forall x,x' \in \mathcal{X}$, to which our results also apply.


Despite its widespread applications, the theoretical understanding of the online gradient descent algorithms are still not satisfactory in the following three aspects.
Firstly, boundedness assumptions on the iterates are often imposed in the literature, which may be violated in practical implementations if the underlying domain is not bounded.
Although a projection of iterates onto a bounded domain guarantees the boundedness assumption, the projection operator may be time-consuming and this introduces an additional challenging problem of tuning the size of the domain.
Secondly, most of the theoretical results are stated in expectation, while we are sometimes more interested in either almost sure convergence or convergence rates with high probability. Indeed, an algorithm may suffer from a high variability and should be used with caution if neither almost sure convergence nor high-probability bounds hold~\citep{shamir2013stochastic}. In particular, an almost sure convergence is still lacking for online gradient descent algorithms applied to general convex problems~\citep{ying2017unregularized}. Lastly, most existing convergence rates are stated for some average of iterates. Though taking average of iterates can improve the robustness of the solution~\citep{nemirovski2009robust}, it can either destroy the sparsity of the solution which is crucial for a proper interpretation of models in many applications, or slow down the training speed in practical implementations~\citep{rakhlin2012making}. 

In this paper, we aim to take a further step to tackle the above mentioned problems. We establish a general sufficient condition and a necessary condition on the step sizes for the convergence of online gradient descent algorithms in expectation. With Doob's martingale convergence theorem and the Borel-Cantelli lemma, a sufficient condition for the almost sure convergence and explicit convergence rates with probability one are also established. Furthermore, we present high-probability bounds for both averaged iterates and the last iterate of online gradient descent algorithms.
To our best knowledge, this is the first high-probability convergence rate for the last iterate of online gradient descent algorithms in the general convex setting.
Our analysis does not impose any boundedness assumptions on the iterates. Indeed, we show that, although implemented in an unbounded domain, the iterates produced by \eqref{OKL} fall into a bounded domain with high probability (up to logarithmic factors).
Our analysis is performed by viewing the one-step progress of online gradient descent algorithms from different yet unified perspectives: one in terms of generalization errors and one in terms of RKHS distances. For both viewpoints, we relate the one-step progress to a martingale difference sequence and a negative term due to the descent nature of the algorithm. Our novelty is to show that the dominant variance term appearing in the application of a Bernstein-type inequality to these martingales can be cancelled out by the negative terms in the one-step progress inequalities. Both viewpoints of the one-step progress are indispensable in our analysis.

The remaining parts of this paper are organized as follows. We present main results in Section \ref{sec:results}. Discussions and comparisons with related work are given in Section \ref{sec:discussion}. The proofs of main results are given in Section \ref{sec:proofs}.


\section{Main Results}\label{sec:results}
Our convergence rates are stated for generalization errors, which, for a prediction rule $f:\xcal\to\rbb$,
are defined as the expected error $\ecal(f)=\int_\mathcal{Z}\phi(y,f(x))d\rho$ incurred from using $f$ to perform prediction.
Our analysis requires to impose mild assumptions on the loss functions.
\begin{assumption}\label{ass:loss}
We assume the loss function $\phi:\ycal\times\rbb\to\rbb_+$ is convex and differentiable with respect to the second argument. Let $\alpha\in(0,1]$ and $L>0$ be two constants. We assume that the gradients of $\phi$ are $(\alpha,L)$-H\"older continuous in the sense
 \begin{equation}\label{loss-holder}
   |\phi'(y,s)-\phi'(y,\tilde{s})|\leq L|s-\tilde{s}|^\alpha,\quad\forall s,\tilde{s}\in\rbb,\forall y\in\ycal.
 \end{equation}
\end{assumption}
We say $\phi$ is smooth if it satisfies \eqref{loss-holder} with $\alpha=1$.
Loss functions satisfying Assumption \ref{ass:loss} are wildly used in machine learning. Smooth loss functions include the least squares loss $\phi(y,a)=\frac{1}{2}(y-a)^2$ and the Huber loss $\phi(y,a)=\frac{1}{2}(y-a)^2$ if $|y-a|\leq1$ and $|y-a|-\frac{1}{2}$ otherwise for regression, as well as the logistic loss $\phi(y,a)=\log(1+\exp(-ya))$ and the quadratically smoothed hinge loss $\phi(y,a)=\max\{0,1-ya\}^2$ for classification~\citep{zhang2004solving}. If $p\in(1,2]$, both the $p$-norm hinge loss $\phi(y,a)=\max\{0,1-ya\}^p$ for classification and the $p$-th power absolute distance $\phi(y,a)=|y-a|^p$ for regression satisfy \eqref{loss-holder} with $\alpha=p-1$~\citep{chen2004support,steinwart2008support}.

Throughout this paper, we assume that a minimizer $f_H=\arg\min_{f\in H_K}\ecal(f)$ exists in $H_K$. We also assume $$\max\big\{\sup_{y\in\ycal}\phi(y,0),\sup_{z\in\zcal}\phi(y,f_H(x))\big\}<\infty \mbox{ and } \kappa:=\sup_{x\in\xcal}\sqrt{K(x,x)}<\infty,$$
this assumption is satisfied if the sample space $\zcal$ is bounded. Denote $\|\cdot\|$ as the norm in $H_K$.
We always use the notation $A_t=\ebb[\ecal(f_t)]-\ecal(f_H)$ and $\hat{A}_t=\ecal(f_t)-\ecal(f_H),\forall t\in\nbb$ for brevity, which are referred to as the expected excess generalization errors and excess generalization errors, respectively.

\medskip

In the following, we present the main results of this paper. We consider three types of convergence: convergence in expectation, almost sure convergence and convergence rates with high probability.

\subsection{Convergence in expectation}\label{sec:conv-exp}
The first part of our main results to be proved in Section \ref{sec:proof-conv-exp} establishes a general sufficient condition (Theorem \ref{thm:suff}) and a necessary condition (Theorem \ref{thm:nece}) on the step size sequence $\{\eta_t\}_{t\in\mathbb{N}}$ for the convergence of $A_t$ to zero.
\begin{theorem}\label{thm:suff}
  Let $\{f_t\}_{t\in\nbb}$ be the sequence produced by \eqref{OKL} and suppose Assumption \ref{ass:loss} holds with $\alpha\in(0,1]$. If
  \begin{equation}\label{step-suff}
    \sum_{t=1}^{\infty}\eta_t=\infty\quad\text{and}\quad\lim_{t\to\infty}\eta_t^\alpha\sum_{k=1}^{t}\eta_k^2=0,
  \end{equation}
  then $\lim_{t\to\infty}\ebb[\ecal(f_t)]-\ecal(f_H)=0$.
\end{theorem}
\begin{theorem}\label{thm:nece}
  Let $\{f_t\}_{t\in\nbb}$ be the sequence produced by \eqref{OKL}.
  Suppose that for any $y\in\ycal$, the function $\phi(y,\cdot):\rbb\to\rbb_+$ is convex and its derivative $\phi'(y,\cdot)$ is $(1,L)$-H\"older continuous. Assume that the step size sequence satisfies $\eta_t\leq 1/(6L\kappa^2),\forall t\in\nbb$ and $\ecal(f_1)\neq\ecal(f_H)$. If $\lim_{t\to\infty}\ebb[\ecal(f_t)]=\ecal(f_H)$, then $\sum_{t=1}^{\infty}\eta_t=\infty$.
\end{theorem}
\begin{remark}
  We now illustrate the above theorems by considering the polynomially decaying step sizes $\eta_t=\eta_1t^{-\theta},t\in\nbb,\theta\geq0$. The condition $\sum_{t=1}^{\infty}\eta_t=\infty$ requires $\theta\leq1$, while the condition $\lim_{t\to\infty}\eta_t^\alpha\sum_{k=1}^{t}\eta_k^2=0$ requires $\theta>\frac{1}{2+\alpha}$. Therefore, Theorem \ref{thm:suff} shows that the iteration scheme \eqref{OKL} with $\eta_t=\eta_1 t^{-\theta}$ and $\theta\in\big(\frac{1}{2+\alpha},1\big]$ guarantees the convergence of $\{A_t\}_{t\in\nbb}$. Theorem \ref{thm:nece} shows that the condition $\theta\leq1$ is also necessary for the convergence.
\end{remark}

\subsection{Almost sure convergence}\label{sec:conv-ae}
The second part of our main results focuses on a sufficient condition (Theorem \ref{thm:conv-ae}) for the almost sure convergence of $\{\hat{A}_t\}_{t\in\nbb}$ to zero and convergence rates with probability $1$ (Theorem \ref{thm:ae-borel}). The proofs of results in this section can be found in Section \ref{sec:proof-conv-ae}.
\begin{theorem}\label{thm:conv-ae}
  Let $\{f_t\}_{t\in\nbb}$ be the sequence given by \eqref{OKL}. If Assumption \ref{ass:loss} holds with $\alpha\in(0,1]$ and the step size sequence satisfies
  \begin{equation}\label{step-ae}
    \sum_{t=1}^{\infty}\eta_t=\infty\quad\text{and}\quad\sum_{t=1}^{\infty}\eta_t^{1+\alpha}<\infty,
  \end{equation}
  then $\lim_{t\to\infty}\ecal(f_t)=\ecal(f_H)$ almost surely.
\end{theorem}
\begin{remark}
  According to Theorem \ref{thm:conv-ae}, we know that $\{\hat{A}_t\}_{t\in\nbb}$ would converge almost surely to $0$ if we consider either the step sizes $\eta_t=\eta_1t^{-\theta}$ with $\theta\in(\frac{1}{1+\alpha},1]$ or the step sizes $\eta_t=\eta_1(t\log^\beta t)^{-\frac{1}{1+\alpha}}$ with $\beta>1$. Specifically, if the loss function is smooth, then we can choose either $\eta_t=\eta_1t^{-\theta}$ with $\theta\in(\frac{1}{2},1]$ or $\eta_t=\eta_1\big(t\log^\beta t\big)^{-\frac{1}{2}}$ with $\beta>1$ to guarantee the convergence of the algorithm \eqref{OKL} almost surely in the sense of generalization errors.
\end{remark}

\begin{theorem}\label{thm:ae-borel}
   Suppose that Assumption \ref{ass:loss} holds with $\alpha\in(0,1]$. Let $\{f_t\}_{t\in\nbb}$ be the sequence given by \eqref{OKL} with $\eta_t=\eta_1t^{-\theta},\theta\in(\frac{1}{\alpha+1},1)$ and $\eta_1\leq\frac{1}{A\kappa^2}$ ($A$ is defined in \eqref{A-B}).
  Then for any $\epsilon>0$,
  \begin{equation}\label{ae-borel}
    \lim_{t\to\infty}t^{\min\{(1-\theta),(\alpha+1)\theta-1\}-\epsilon}\hat{A}_t=0 \text{ almost surely.}
  \end{equation}
  Specifically, if we choose $\theta=\frac{2}{2+\alpha}$, then $\lim_{t\to\infty}t^{\frac{\alpha}{2+\alpha}-\epsilon}\hat{A}_t=0$ almost surely.
\end{theorem}

\subsection{Convergence rates with high probability}\label{sec:conv-prob}
The last part of our main results is on high-probability bounds for the excess generalization errors, the proof of which is given in Section \ref{sec:proof-conv-prob}.
With high probability, Theorem \ref{thm:conv-prob} establishes the boundedness (up to logarithmic factors) of the weighted summation $\sum_{t=1}^{T}\eta_t\hat{A}_t$,
from which the decay rate of the excess generalization error  $\ecal(\bar{f}_T^\eta)-\ecal(f_H)$ associated to a weighted average of the iterates $\bar{f}_T^\eta:=\frac{\sum_{t=1}^{T}\eta_tf_t}{\sum_{t=1}^{T}\eta_t}$ follows directly.

\begin{theorem}\label{thm:conv-prob}
  Let $\{f_t\}_{t\in\nbb}$ be the sequence given by \eqref{OKL}.
  Suppose that Assumption \ref{ass:loss} holds with $\alpha\in(0,1]$. Assume the step size sequence satisfies $\eta_t\leq\frac{1}{A\kappa^2},\eta_{t+1}\leq\eta_t$ for all $t\in\nbb$
  and $\sum_{t=1}^{\infty}\eta_t^2<\infty$. Then, there exists a constant $\widetilde{C}$ independent of $T$ (explicitly given in the proof) such that for any $\delta\in(0,1)$ the following inequality holds with probability at least $1-\delta$
  \begin{equation}\label{conv-prob}
    \sum_{t=1}^{T}\eta_t\big[\ecal(f_t)-\ecal(f_H)\big] \leq \widetilde{C}\log^{\frac{3}{2}}\frac{2T}{\delta}\quad\text{and}\quad\ecal(\bar{f}_T^\eta)-\ecal(f_H)\leq \frac{\widetilde{C}\log^{\frac{3}{2}}\frac{2T}{\delta}}{\sum_{t=1}^{T}\eta_t}.
  \end{equation}
\end{theorem}
\begin{remark}
  For the step size sequence $\eta_t=\eta_1t^{-\theta},\theta>\frac{1}{2}$, Theorem \ref{thm:conv-prob} implies that $\ecal(\bar{f}_T^\eta)-\ecal(f_H)=O\big(T^{\theta-1}\log^{\frac{3}{2}}\frac{T}{\delta}\big)$ with probability at least $1-\delta$. If we consider $\eta_t=\eta_1\big(t\log^\beta t\big)^{-\frac{1}{2}}$ with $\beta>1$, then with probability $1-\delta$ we have $\ecal(\bar{f}_T^\eta)-\ecal(f_H)=O\big(T^{-\frac{1}{2}}\log^{\frac{3+\beta}{2}}\frac{T}{\delta}\big)$.
\end{remark}

A key feature of Theorem \ref{thm:conv-prob} distinguishing it from the existing results is that it avoids boundedness assumptions on the iterates, which are always imposed in the literature~\citep{nemirovski2009robust,duchi2010composite}. Indeed, an essential ingredient in proving Theorem \ref{thm:conv-prob} is to show that $\{f_t\}_{t\in\nbb}$ produced by \eqref{OKL} would fall into a bounded ball of $H_K$ (up to logarithmic factors) with high probability, as shown in the following proposition.

\begin{proposition}\label{prop:iterate-bound-pr}
  Suppose assumptions in Theorem \ref{thm:conv-prob} hold.
  Then, there exists a constant $\bar{C}\geq1$ independent of $T$ (explicitly given in the proof) such that for any $\delta\in(0,1)$ the following inequality holds with probability at least $1-\delta$
  $$
    \max_{1\leq t\leq T}\|f_t-f_H\|^2\leq\bar{C}\log\frac{T}{\delta}.
  $$
\end{proposition}
A key ingredient to prove Proposition \ref{prop:iterate-bound-pr} is to establish the following one-step progress inequality in terms of the RKHS distances (see \eqref{iterate-prob-0})
$$
  \|f_{t+1}-f_H\|^2 \leq \|f_t-f_H\|^2+C\eta_t^2 + 2\eta_t\big(\ecal(f_H)-\ecal(f_t)\big)+\xi_t,
$$
where $C$ is a constant and $\{\xi_t\}_{t\in\nbb}$ is a Martingale difference sequence. Our novelty in applying a Bernstein-type inequality to control the martingale $\sum_{t=1}^{T}\xi_t$ is to show that the associated variances can be cancelled out by the negative term $2\sum_{t=1}^{T}\eta_t\big(\ecal(f_H)-\ecal(f_t)\big)$.
Although Theorem \ref{thm:conv-prob} only considers the behavior of the weighted average $\bar{f}_T^\eta$ of iterates, it is possible to establish similar convergence rates for the uniform average of iterates $\bar{f}_T:=\frac{1}{T}\sum_{t=1}^{T}f_t$ (Proposition \ref{prop:unif-average}).


Theorem \ref{thm:last} establishes a general high-probability bound for the excess generalization error of the last iterate in terms of the step size sequence.
\begin{theorem}\label{thm:last}
Suppose that the assumptions in Theorem \ref{thm:conv-prob} hold. Then, there exists a constant $\widetilde{C}'$ independent of $T$ (explicitly given in the proof) such that for any $\delta\in(0,1)$ the following inequality holds with probability at least $1-\delta$
  \begin{equation}\label{last}
    \ecal(f_{T+1})-\ecal(f_H)\leq\widetilde{C}'\max\Big\{\big[\sum_{t=\half}^{T}\eta_t\big]^{-1},\eta_{\half}^\alpha,\sum_{t=\half}^{T}\eta_t^{1+\alpha}\Big\}\log^2\frac{3T}{\delta},
  \end{equation}
  where $\lfloor \frac{T}{2} \rfloor$ denotes the largest integer not greater than $\frac{T}{2}$.
\end{theorem}

 To establish high-probability error bounds for the last iterate of online gradient descent algorithm is an interesting problem which is not well studied, to our best knowledge, in the general convex setting. The key ingredient in our analysis is the following one-step progress inequality in terms of generalization errors (see \eqref{last-one-step-progress})
$$
  \hat{A}_{t+1}\leq \hat{A}_t-\eta_t\|\nabla\ecal(f_t)\|^2+\bar{\xi}_t+C\eta_t^{1+\alpha},
$$
where $C$ is a constant and $\{\bar{\xi}_t\}$ is a martingale difference sequence. A key observation of our analysis is that the variance of the martingale $\sum_{t=1}^{T}\bar{\xi}_t$ can be cancelled out by the negative term $-\sum_{t=1}^{T}\eta_t\|\nabla\ecal(f_t)\|^2$ in the above one-step progress inequality, paving the way for the application of a Bernstein-type inequality for martingales.

We can derive explicit convergence rates in Corollary \ref{cor:rate-last} by considering polynomially decaying step sizes in Theorem \ref{thm:last}.
\begin{corollary}\label{cor:rate-last}
  Let $\{f_t\}_{t\in\nbb}$ be the sequence given by \eqref{OKL} with $\eta_t=\eta_1t^{-\theta},\theta\in(\frac{1}{2},1)$ and $\eta_1\leq\frac{1}{A\kappa^2}$
  If Assumption \ref{ass:loss} holds and $\delta\in(0,1)$, then the following inequality holds with probability $1-\delta$
  $$
    \ecal(f_{T+1})-\ecal(f_H)=O\Big(T^{\max\big\{\theta-1,1-(1+\alpha)\theta\big\}}\log^2\frac{T}{\delta}\Big).
  $$
  If we choose $\theta=\frac{2}{2+\alpha}$, then with probability at least $1-\delta$ we derive $\ecal(f_{T+1})-\ecal(f_H)=O\Big(T^{-\frac{\alpha}{2+\alpha}}\log^2\frac{T}{\delta}\Big).$
\end{corollary}
\begin{remark}
  It should be mentioned that, unlike Theorem \ref{thm:conv-prob}, the convergence rates in Corollary \ref{cor:rate-last} depend on the smoothness parameter $\alpha$ and is not able to attain the minimax optimal convergence rate $O(T^{-\frac{1}{2}})$~\citep{agarwal2009information}. Indeed, for smooth loss functions, Corollary \ref{cor:rate-last} establishes the convergence rate $O\big(T^{-\frac{1}{3}}\log^2\frac{T}{\delta}\big)$ with high probability, which matches the bounds in-expectation $A_T=O(T^{-\frac{1}{3}})$ up to logarithmic factors established in \citet{moulines2011non,ying2017unregularized}. It remains a challenging problem to further improve the high-probability bounds for $\hat{A}_T$.
\end{remark}

\section{Discussions}\label{sec:discussion}
In this section, we discuss related work on convergence of online/stochastic gradient descent algorithms from three viewpoints: convergence in expectation, almost sure convergence and convergence rates with high probability.

\subsection{Related work on convergence in expectation}
Most studies of online gradient descent algorithms focus on convergence in expectation \citep{zhang2004solving,ying2006online,duchi2009efficient,shamir2013stochastic,lin2016generalization,hardt2016train,ying2017unregularized}. Convergence rates $O(T^{-\frac{1}{2}})$ were established for some averaged iterates produced by \eqref{OKL} in a parametric setting with the linear kernel $K_x=x$~\citep{zhang2004solving}. These results were extended to online gradient descent algorithms in RKHSs with the specific least squares loss function \citep{ying2008online,dieuleveut2016nonparametric}, and online mirror descent algorithms performing updates in Banach spaces \citep{duchi2010composite}. Under boundedness assumptions on the iterates and (sub)gradients, the convergence rate $O(T^{-\frac{1}{2}}\log T)$ was established for the expected excess generalization error of the last iterate~\citep{shamir2013stochastic}. 
Recently, a general condition on the step sizes as \eqref{step-ae} was established for the convergence of the algorithm \eqref{OKL}, in the sense $\lim_{t\to\infty}A_t=0$, with loss functions satisfying Assumption \ref{ass:loss}~\citep{ying2017unregularized}. This sufficient condition is stricter than our condition \eqref{step-suff}. To see this clearly, we consider the polynomially decaying step sizes $\eta_t=\eta_1t^{-\theta}$, for which the condition \eqref{step-ae} requires $\theta\in(\frac{1}{1+\alpha},1]$ while our condition \eqref{step-suff} requires $\theta\in(\frac{1}{2+\alpha},1]$.
Furthermore, our discussion also implies a necessary condition for the convergence in expectation.

Implemented in either a parametric or a nonparametric setting, regularized online learning algorithms have also received considerable attention \citep{kivinen2004online,smale2006online,ying2006online,smale2009online}, which differ from \eqref{OKL} by introducing a regularization term to avoid overfitting. This algorithm updates iterates as follows
\begin{equation}\label{okl-regularized}
  f_{t+1}=(1-\lambda\eta_t)f_t-\eta_t\phi'(y_t,f_t(x_t))K_{x_t},
\end{equation}
where $\lambda>0$ is a regularization parameter and the term $\lambda f_t+\phi'(y_t,f_t(x_t))K_{x_t}$ is used as an unbiased estimator of the gradient for the regularized generalization error $\ecal^\lambda(f):=\ecal(f)+\frac{\lambda}{2}\|f\|^2$ at $f=f_t$. Convergence rates in expectation can be stated for either the excess regularized generalization error $\ecal^\lambda(f_T)-\ecal^\lambda(f_\lambda)$~\citep{shamir2013stochastic} or the RKHS distance $\|f_T-f_\lambda\|$~\citep{smale2006online,ying2006online,yao2010complexity},
where $f_\lambda=\arg\min_{f\in H_K}\ecal^\lambda(f)$ is the minimizer of the regularized generalization error.
When the loss function is smooth, a sufficient and necessary condition as
\begin{equation}\label{step-ae-strong-convex}
  \lim_{t\to\infty}\eta_t=0\quad\text{and}\quad\sum_{t=1}^{\infty}\eta_t=\infty
\end{equation}
was recently established for the convergence of $\{\ebb[\|f_t-f_\lambda\|^2]\}_{t\in\nbb}$ to zero~\citep{lei2017convergence}.
A disadvantage of the regularization scheme \eqref{okl-regularized} is that it requires to tune two sequence of hyper-parameters: a regularization parameter and the step sizes. As a comparison, an implicit regularization can be attained in the unregularized scheme \eqref{OKL} by tuning only the step sizes.

\subsection{Related work on almost sure convergence}
Existing almost sure convergence of online learning algorithm is mainly stated for the RKHS distances, which require to impose some type of strong convexity assumption on the objective function $\ecal(f)$. In the parametric setting with the learning scheme \eqref{OKL}, a sufficient condition as
$$
  \sum_{t=1}^{\infty}\eta_t=\infty\quad\text{and}\quad\sum_{t=1}^{\infty}\eta_t^2<\infty
$$
was established for the almost sure convergence of $\|f_t-f_H\|^2$ if the objective function attains a unique minimizer and satisfies~\citep{bottou1998online}
$$\inf_{\|f-f_H\|^2>\epsilon}\inn{f-f_H,\nabla \ecal(f)}>0,\; \forall \epsilon>0\quad\text{and}\quad\|\nabla \ecal(f)\|^2\leq \widetilde{A}+\widetilde{B}\|f-f_H\|^2,\; \forall f\in H_K,$$
where $\widetilde{A}$ and $\widetilde{B}$ are two constants.
This result was extended to the online mirror descent setting under some convexity assumption on the objective function measured by Bregman distances induced by the associated mirror map~\citep{lei2017convergence}.
For polynomially decaying step sizes $\eta_t=\eta_1t^{-\theta}$ with $\theta\in(0,1)$, almost sure convergence of $\|f_t-f_\lambda\|$ was shown for regularized online learning algorithms \eqref{okl-regularized} specified to the least squares loss function~\citep{yao2010complexity}. The analysis in \citet{yao2010complexity} roots its foundation on the martingale decompositions of the reminders $f_t-f_\lambda$, which only holds in the least squares regularization setting.
Almost sure convergence was recently studied for the randomized Kaczmarz algorithm \citep{lin2015learning}, which is an instantiation of \eqref{OKL} with $\phi(y,a)=\frac{1}{2}(y-a)^2$ and $K_x=x$. The analysis there heavily depends on a restricted strong convexity of the objective function in a linear subspace where the learning takes place, which can not apply to general loss functions.
As compared to the above mentioned results, our almost sure convergence is stated for the excess generalization errors with general loss functions and requires no assumptions on the strong convexity of the objective function $\ecal(f)$.


\subsection{Related work on convergence rates with high probability}
In this section, we survey related work on convergence rates with high probability. We divide our discussions into two parts according to the convexity of the objective function.

As far as we know, all existing high-probability convergence rates of online gradient descent algorithms with general convex functions focus on some average of iterates (here we are not interested in probabilistic bounds with a polynomial dependence on $1/\delta$). The following online projected gradient descent algorithm was studied in \citet{nemirovski2009robust,duchi2010composite}
\begin{equation}\label{OKL-projected}
  f_{t+1}=\text{Proj}_{\widetilde{H}}\Big[f_t-\eta_t\phi'(y_t,f_t(x_t))K_{x_t}\Big],
\end{equation}
where $\widetilde{H}$ is a compact subset of $H_K$ and $\text{Proj}_{\widetilde{H}}(f)=\arg\min_{\tilde{f}\in\widetilde{H}}\|f-\tilde{f}\|$ is the projection of $f$ onto $\widetilde{H}$. Under the boundedness assumption
$$
  \ebb\Big[\exp\big[\|\phi'(y,f(x))K_x\|^2/G^2\big]\Big]\leq \exp(1)\quad\forall f\in\widetilde{H},
$$
it was shown that the weighted average $\bar{f}_T^\eta=\frac{\sum_{t=1}^{T}\eta_tf_t}{\sum_{t=1}^{T}\eta_t}$ of iterates produced by \eqref{OKL-projected} with a constant step size satisfying the following inequality with probability $1-\delta$
$$
  \ecal(\bar{f}_T^\eta)-\ecal(f_H)=O\big(GDT^{-\frac{1}{2}}\log\delta^{-1}\big),
$$
where $D=\sup_{f,\tilde{f}\in\widetilde{H}}\|f-\tilde{f}\|$ is the diameter of the subspace $\widetilde{H}$. Under a stronger assumption $\|\phi'(y,f(x))K_x\|\leq G$ for all $(x,y)\in\zcal,f\in\widetilde{H}$, the uniform average $\bar{f}_T=\frac{1}{T}\sum_{t=1}^{T}f_t$ of iterates produced by \eqref{OKL-projected} with step sizes $\eta_t=\eta_1t^{-\frac{1}{2}}$ was shown to enjoy the bound $\ecal(\bar{f}_T)-\ecal(f_H)=O(DGT^{-\frac{1}{2}}\log^{\frac{1}{2}}\frac{1}{\delta})$ with probability at least $1-\delta$. In comparison with these results, the convergence rates in Theorem \ref{thm:conv-prob} are derived without the projection step and any boundedness assumption on the gradients. Indeed, most of the efforts in proving Theorem \ref{thm:conv-prob} is to show $\|f_t-f_H\|^2=O(\log\frac{T}{\delta})$ with probability at least $1-\delta$. It is implied that the possibly computationally expensive projection step can be removed without harming the behavior of the online gradient descent algorithms. Furthermore, Theorem \ref{thm:last} gives, to our best knowledge, the first high-probability bounds for the last iterate of online gradient descent algorithms in the general convex setting. A framework to transfer regret bounds of online learning algorithms to high-probability bounds for the uniform average of iterates was established by \citet{cesa2004generalization}.


Now we review some high-probability studies for online gradient descent algorithms in the strongly convex setting, for which some results for the last iterate can be found in the literature.
For the online regularized algorithm \eqref{okl-regularized} with the least squares loss function and $\eta_t=\eta_1t^{-\theta},\theta\in[0,1)$, the following inequality was derived with probability at least $1-\delta$~\citep{yao2010complexity}
$$
  \|f_T-f_\lambda\|^2=O\Big(\lambda^{-2+\frac{1}{1-\theta}}T^{-\theta}\log\frac{1}{\delta}\Big).
$$
The analysis in \citet{yao2010complexity} is based on an integral operator approach, which can not be extended to general loss functions. Under almost sure boundedness assumption $\|(\phi'(y_t,f_t(x_t))+\lambda)K_{x_t}\|\leq G$ for all $t\in\nbb$, the following improved bound for the last iterate of \eqref{okl-regularized} with general loss functions and step sizes $\eta_t=\eta_1(t\lambda)^{-1}$ was established with probability at least $1-\delta$~\citep{rakhlin2012making}
\begin{equation}\label{rahklin-rkhs}
  \|f_T-f_\lambda\|^2=O\Big(G^2\lambda^{-2}T^{-1}\log\frac{\log T}{\delta}\Big).
\end{equation}
Although this bound enjoys a tight dependence on $T$, its dependence on the regularization parameter $\lambda$ is suboptimal. To make a clear comparison between this result and ours, we consider here the specific least squares loss function and assume that the regression function $f_\rho(x):=\ebb[Y|X=x]$ belongs to $H_K$. In this case, Lemma \ref{lem:smooth-hilbert} translates \eqref{rahklin-rkhs} to the following high-probability bounds on excess generalization errors
\begin{equation}\label{rakhlin-gene}
  \ecal(f_T)+\frac{\lambda}{2}\|f_T\|^2=\ecal(f_\lambda)+\frac{\lambda}{2}\|f_\lambda\|^2+O\Big(G^2\lambda^{-2}T^{-1}\log\frac{\log T}{\delta}\Big).
\end{equation}
The assumption $f_\rho\in H_K$ implies  $D(\lambda):=\ecal(f_\lambda)-\ecal(f_\rho)+\frac{\lambda}{2}\|f_\lambda\|^2=O(\lambda)$ \citep{cucker2007learning} and therefore \eqref{rakhlin-gene} reads as
\begin{align*}
  \ecal(f_T)-\ecal(f_\rho) &= \Big(\ecal(f_T)-\ecal(f_\lambda)-\frac{\lambda}{2}\|f_\lambda\|^2\Big)+\Big(\ecal(f_\lambda)-\ecal(f_\rho)+\frac{\lambda}{2}\|f_\lambda\|^2\Big)\\
  &=O\Big(G^2\lambda^{-2}T^{-1}\log\frac{\log T}{\delta}\Big)+O(\lambda).
\end{align*}
If we choose $\lambda=c\big(G^2T^{-1}\log\frac{\log T}{\delta}\big)^{\frac{1}{3}}$ for a constant $c>0$, then the above inequality translates to $\ecal(f_T)-\ecal(f_\rho)=O\Big(\big(G^2T^{-1}\log\frac{\log T}{\delta}\big)^{\frac{1}{3}}\Big)$, which matches our convergence rates up to logarithmic factors. Note that the regularization parameter $\lambda$ needs to be tuned according to $T$ to balance the bias and variance in \eqref{rakhlin-gene}, which may not be accessible in practical implementations. To deal with this issue, a class of fully online regularized algorithms is proposed and investigated by allowing the regularization parameters to vary along the learning process~\citep{ye2007fully,tarres2014online}.
As a comparison, without a regularization parameter to tune, the unregularized online learning algorithm \eqref{OKL} achieves a bias-variance balance by tuning only the step sizes.
Furthermore, the convergence rates \eqref{rahklin-rkhs} require to impose the non-intuitive boundedness assumptions on the gradients encountered during the iterations, which may be violated in practical implementations. This boundedness assumption is removed in our analysis.

\section{Proofs}\label{sec:proofs}
In this section, we present the proofs for the results given in Section \ref{sec:results}.
Our discussions require to use a property established in the following lemma on functions with $(\alpha,L)$-H\"older continuous gradients.
This lemma is motivated by similar results in the literature~\citep[see, e.g.,][]{ying2017unregularized} and we present the proof in Section \ref{sec:additional-proofs} for completeness.
\begin{lemma}\label{lem:smooth-hilbert}
  Let $H$ be a Hilbert space associated with the inner product $\inn{\cdot,\cdot}$. Let $\gcal:H\to\rbb$ be a differentiable functional satisfying
  $$
    \|\nabla\gcal(f)-\nabla\gcal(\tilde{f})\|\leq L\|f-\tilde{f}\|^\alpha,\quad\forall f,\tilde{f}\in H,
  $$
  where $L>0,\alpha\in(0,1],\nabla$ is the gradient operator and $\|\cdot\|$ is the norm induced by the inner product. Then, the following inequality holds for any $f,\tilde{f}\in H$
  \begin{equation}\label{smooth-hilbert}
    \frac{\alpha\|\nabla\gcal(f)-\nabla\gcal(\tilde{f})\|^{\frac{1+\alpha}{\alpha}}}{(1+\alpha)L^{\frac{1}{\alpha}}}
    \leq\gcal(f)-\big[\gcal(\tilde{f})+\inn{f-\tilde{f},\nabla\gcal(\tilde{f})}\big]\leq\frac{L\|f-\tilde{f}\|^{1+\alpha}}{1+\alpha}.
  \end{equation}
\end{lemma}

With Lemma \ref{lem:smooth-hilbert}, we can derive the following lemma on gradients of loss functions at iterates of the algorithm \eqref{OKL}. Its power consists in bounding the gradients for the possibly unbounded iterates $\{f_t\}_{t\in\nbb}$ by the gradients for $f_H$ and the excess generalization errors, the first of which can be considered as a constant while the second of which is exactly the term we are interested in. For a random variable $z$, we use $\ebb_z[\cdot]$ to denote the conditional expectation with respect to $z$.
\begin{lemma}\label{lem:grad-bound}
    Suppose Assumption \ref{ass:loss} holds and $\beta\in(0,1]$. Then,
    \begin{multline}\label{grad-bound}
      \ebb_{z_t}\big[|\phi'(y_t,f_t(x_t))|^{1+\beta}\big]\leq
      2^\beta L^{\frac{1}{\alpha}}(1+\beta)\big[\ecal(f_t)-\ecal(f_H)\big]+\\
      \frac{2^\beta(1-\alpha\beta)}{1+\alpha}+2^\beta\ebb_{z_t}\big[|\phi'(y_t,f_H(x_t))|^{1+\beta}\big],\quad\forall t\in\nbb.
    \end{multline}
\end{lemma}
\begin{proof}
With the elementary inequality $|u+v|^{1+\beta}\leq 2^\beta[|u|^{1+\beta}+|v|^{1+\beta}]$ and the Young's inequality
\begin{equation}\label{young}
  uv\leq p^{-1}|u|^p+q^{-1}|v|^q,\quad\forall u,v\in\rbb,p^{-1}+q^{-1}=1,p\geq0,
\end{equation}
the term $|\phi'(y_t,f_t(x_t))|^{1+\beta}$ can be controlled by
\begin{align}
   & |\phi'(y_t,f_t(x_t))|^{1+\beta}\leq \Big[|\phi'(y_t,f_t(x_t))-\phi'(y_t,f_H(x_t))|+|\phi'(y_t,f_H(x_t))|\Big]^{1+\beta} \notag\\
   & \leq 2^\beta|\phi'(y_t,f_t(x_t))-\phi'(y_t,f_H(x_t))|^{1+\beta}+2^\beta|\phi'(y_t,f_H(x_t))|^{1+\beta} \notag\\
   & \leq \frac{2^\beta\alpha(1+\beta)}{1+\alpha}|\phi'(y_t,f_t(x_t))-\phi'(y_t,f_H(x_t))|^{\frac{1+\alpha}{\alpha}}+\frac{2^\beta(1-\alpha\beta)}{1+\alpha}+2^\beta|\phi'(y_t,f_H(x_t))|^{1+\beta}.\label{grad-2}
\end{align}
It follows from the first inequality of \eqref{smooth-hilbert} that
\begin{multline*}
  \frac{\alpha}{1+\alpha} |\phi'(y_t,f_t(x_t))-\phi'(y_t,f_H(x_t))|^{\frac{1+\alpha}{\alpha}}\leq \\ L^{\frac{1}{\alpha}}\Big[\phi(y_t,f_t(x_t))-\phi(y_t,f_H(x_t))-\phi'(y_t,f_H(x_t))(f_t(x_t)-f_H(x_t))\Big].
\end{multline*}
Plugging the above inequality into \eqref{grad-2} and taking expectations with respect to $z_t$ (note $f_t$ is independent of $z_t$), we get
\begin{align*}
  \ebb_{z_t}\big[|\phi'(y_t,f_t(x_t))&|^{1+\beta}\big] \leq 2^\beta L^{\frac{1}{\alpha}}(1+\beta)\Big[\ecal(f_t)-\ecal(f_H)-\big\langle f_t-f_H,\ebb_{z_t}\big[\phi'(y_t,f_H(x_t))K_{x_t}\big]\big\rangle\Big]\\
  &\qquad\qquad\qquad\qquad+\frac{2^\beta(1-\alpha\beta)}{1+\alpha}+2^\beta\ebb_{z_t}\big[|\phi'(y_t,f_H(x_t))|^{1+\beta}\big] \\
  & = 2^\beta L^{\frac{1}{\alpha}}(1+\beta)\big[\ecal(f_t)-\ecal(f_H)\big]+\frac{2^\beta(1-\alpha\beta)}{1+\alpha}+2^\beta\ebb_{z_t}\big[|\phi'(y_t,f_H(x_t))|^{1+\beta}\big].
\end{align*}
Here the last identity holds since
$$
\nabla\ecal(f_H)=\ebb_z\big[\phi'(y,f_H(x))K_x\big]=0.$$
The proof is complete.
\end{proof}

\subsection{Proofs for convergence in expectation}\label{sec:proof-conv-exp}
Before proving Theorem \ref{thm:suff} and Theorem \ref{thm:nece} on convergence in expectation, we first present some preparatory results.
Our first preliminary result is a weak result on convergence in expectation under a weak condition on the step size sequence \eqref{step-size}. Eq. \eqref{conv-weak-a} implies the existence of a sub-index sequence $\{i_t\}_{t\in\nbb}$ satisfying $\lim_{t\to\infty}A_{i_t}=0$, while \eqref{conv-weak-b} shows the convergence of a weighted average of the expected excess generalization errors. This result is derived based on a one-step progress inequality in terms of distances in RKHSs (see \eqref{conv-weak-3}).
\begin{proposition}\label{prop:conv-weak}
  Let $\{f_t\}_{t\in\nbb}$ be the sequence given by \eqref{OKL} and suppose Assumption \ref{ass:loss} holds. If
  \begin{equation}\label{step-size}
    \lim_{t\to\infty}\eta_t=0\quad\text{and}\quad\sum_{t=1}^{\infty}\eta_t=\infty,
  \end{equation}
  then
  \begin{equation}\label{conv-weak-a}
    \liminf\limits_{t\to\infty}\ebb[\ecal(f_t)-\ecal(f_H)]=0
  \end{equation}
  and
  \begin{equation}\label{conv-weak-b}
    \lim_{T\to\infty}\big[\sum_{t=1}^{T}\eta_t\big]^{-1}\sum_{t=1}^{T}\eta_t\ebb[\ecal(f_t)-\ecal(f_H)]=0.
  \end{equation}
\end{proposition}
\begin{lemma}\label{lem:series}
  Let $\{\eta_t\}_{t\in\nbb}$ be a sequence of positive numbers. If $\lim_{t\to\infty}\eta_t=0$ and $\sum_{t=1}^{\infty}\eta_t=\infty$, then $\lim_{t\to\infty}\big[\sum_{k=1}^{t}\eta_k\big]^{-1}\sum_{k=1}^{t}\eta_k^2=0$.
\end{lemma}
\begin{proof}[Proof of Proposition \ref{prop:conv-weak}]
According to the iteration strategy \eqref{OKL}, we derive
\begin{align}
  \|f_{t+1}-f_H\|^2 &= \|f_t-\eta_t\phi'(y_t,f_t(x_t))K_{x_t}-f_H\|^2 \notag\\
   & \leq \|f_t-f_H\|^2 + \eta_t^2|\phi'(y_t,f_t(x_t))|^2\kappa^2-2\eta_t\inn{f_t-f_H,\phi'(y_t,f_t(x_t))K_{x_t}} \label{conv-weak-1}\\
   & \leq \|f_t-f_H\|^2 + \eta_t^2|\phi'(y_t,f_t(x_t))|^2\kappa^2+2\eta_t\big[\phi(y_t,f_H(x_t))-\phi(y_t,f_t(x_t))\big].\label{conv-weak-2}
\end{align}
Note that $f_t$ is independent of $z_t$.
Taking expectations with respect to $z_t$ on both sides and using \eqref{grad-bound} with $\beta=1$, we derive
\begin{align*}
   & \ebb_{z_t}\big[\|f_{t+1}-f_H\|^2\big] \leq \|f_t-f_H\|^2 + \eta_t^2\kappa^2\ebb_{z_t}\big[|\phi'(y_t,f_t(x_t))|^2\big] + 2\eta_t[\ecal(f_H)-\ecal(f_t)]\\
   & \leq \begin{multlined}[t][0.85\columnwidth]
   \|f_t-f_H\|^2 + 4\eta_t^2\kappa^2L^{\frac{1}{\alpha}}[\ecal(f_t)-\ecal(f_H)] +\frac{2(1-\alpha)\eta_t^2\kappa^2}{1+\alpha}\\
   + 2\eta_t^2\kappa^2\ebb_{z_t}\big[|\phi'(y_t,f_H(x_t))|^2\big] + 2\eta_t[\ecal(f_H)-\ecal(f_t)]
   \end{multlined} \\
   & = \|f_t-f_H\|^2 + 2\eta_t\big(1-2\eta_t\kappa^2L^{\frac{1}{\alpha}}\big)[\ecal(f_H)-\ecal(f_t)] + 2\eta_t^2\kappa^2\Big(\ebb_{z_t}\big[|\phi'(y_t,f_H(x_t))|^2\big]+\frac{1-\alpha}{1+\alpha}\Big).
\end{align*}
Since $\lim_{t\to\infty}\eta_t=0$, we can find an integer $t_1\in\nbb$ such that $\eta_t\leq \frac{1}{4\kappa^2L^{\frac{1}{\alpha}}},\forall t\geq t_1$. This together with $\ecal(f_H)\leq \ecal(f_t)$ implies
\begin{equation}\label{conv-weak-3}
  \eta_t[\ecal(f_t)-\ecal(f_H)]\leq \|f_t-f_H\|^2-\ebb_{z_t}\big[\|f_{t+1}-f_H\|^2\big] + \gamma\eta_t^2,\quad\forall t\geq t_1,
\end{equation}
where we introduce $\gamma=2\kappa^2\Big(\ebb_{z_t}\big[|\phi'(y_t,f_H(x_t))|^2\big]+\frac{1-\alpha}{1+\alpha}\Big)$.
Taking expectations followed with a summation from $t=t_1$ to $t=T$ gives
$$
  \sum_{t=t_1}^{T}\eta_tA_t \leq \ebb[\|f_{t_1}-f_H\|^2] + \gamma\sum_{t=t_1}^{T}\eta_t^2.
$$
It then follows that
\begin{align*}
  \lim_{T\to\infty}\big[\sum_{t=1}^{T}\eta_t\big]^{-1}\sum_{t=1}^{T}\eta_tA_t & = \lim_{T\to\infty}\big[\sum_{t=1}^{T}\eta_t\big]^{-1}\sum_{t=1}^{t_1-1}\eta_tA_t +
  \lim_{T\to\infty}\big[\sum_{t=1}^{T}\eta_t\big]^{-1}\sum_{t=t_1}^{T}\eta_tA_t\\
   & \leq \lim_{T\to\infty}\big[\sum_{t=1}^{T}\eta_t\big]^{-1}\big[\ebb[\|f_{t_1}-f_H\|^2]  + \gamma\sum_{t=t_1}^{T}\eta_t^2\big]=0,
\end{align*}
where we have used $\lim_{t\to\infty}\big[\sum_{k=1}^{t}\eta_k\big]^{-1}\sum_{k=1}^{t}\eta_k^2=0$ established in Lemma \ref{lem:series}.
This establishes \eqref{conv-weak-b}.

We now prove \eqref{conv-weak-a} by contradiction strategy.
Suppose to the contrary that $\liminf\limits_{t\to\infty}A_t=\tilde{a}>0$. Then, there exists $\tilde{t}\in\nbb$ such that $A_t\geq 2^{-1}\tilde{a},\forall t\geq \tilde{t}$, from which we derive from \eqref{conv-weak-b} that
\begin{align*}
  0=\lim_{T\to\infty}\frac{\sum_{t=1}^{T}\eta_tA_t}{\sum_{t=1}^{T}\eta_t}&\geq \frac{\tilde{a}}{2}\lim_{T\to\infty}\frac{\sum_{t=\tilde{t}+1}^{T}\eta_t}{\sum_{t=1}^{T}\eta_t}
   = \frac{\tilde{a}}{2} - \frac{\tilde{a}}{2}\lim_{T\to\infty}\frac{\sum_{t=1}^{\tilde{t}}\eta_t}{\sum_{t=1}^{T}\eta_t}=\frac{\tilde{a}}{2}.
\end{align*}
This leads to a contradiction. Therefore, $\liminf_{t\to\infty}A_t=0$ and the proof is complete.
\end{proof}

As our second preliminary result, Lemma \ref{lem:iterate-bound} establishes an upper bound on $\ebb[\|f_t-f_H\|_2^2]$ in terms of the step size sequence, as well as a lower bound on $\ebb[\|\nabla \ecal(f_t)\|^2]$ in terms of the step size sequence and the expected excess generalization errors.
\begin{lemma}\label{lem:iterate-bound}
  Let $\{f_t\}_{t\in\nbb}$ be the sequence given by \eqref{OKL}. If Assumption \ref{ass:loss} holds and $\lim_{t\to\infty}\eta_t=0$, then there exist constants $\widehat{C},\gamma>0$ independent of $t$ such that the following inequalities hold for any $t\in\nbb$
  \begin{gather}
    \ebb[\|f_t-f_H\|^2]\leq \widehat{C}+\gamma\sum_{k=1}^{t}\eta_k^2\label{iterate-bound-a}\\
    \intertext{and}
    \ebb[\|\nabla\ecal(f_t)\|^2]\geq \frac{\big(\ebb[\ecal(f_t)-\ecal(f_H)]\big)^2}{\widehat{C}+\gamma\sum_{k=1}^{t}\eta_k^2}.\label{iterate-bound-b}
  \end{gather}
\end{lemma}
\begin{proof}
  Since $\ecal(f_t)\geq\ecal(f_H)$ for all $t\in\nbb$, \eqref{conv-weak-3} implies
  $$
    \ebb[\|f_{t+1}-f_H\|^2]\leq\ebb[\|f_t-f_H\|^2]+\eta_t^2\gamma,\quad\forall t\geq t_1,
  $$
  where $\gamma$ and $t_1$ are defined in the proof of Proposition \ref{prop:conv-weak}. Taking a summation of the above inequality from $t=t_1$ to $t=T$ shows
  $$
    \ebb[\|f_{T+1}-f_H\|^2]\leq\ebb[\|f_{t_1}-f_H\|^2]+\gamma\sum_{t=t_1}^{T}\eta_t^2\leq \widehat{C}+\gamma\sum_{t=1}^{T}\eta_t^2,
  $$
  where we introduce $\widehat{C}=\ebb[\|f_{t_1}-f_H\|^2]$. This establishes \eqref{iterate-bound-a}.

  We now turn to \eqref{iterate-bound-b}. According to the convexity of $\ecal$ and Schwartz inequality, we get
  \begin{align*}
    \ebb[\ecal(f_t)]-\ecal(f_H) & \leq \ebb\big[\big\langle \nabla\ecal(f_t),f_t-f_H\big\rangle\big] \leq \ebb[\|\nabla\ecal(f_t)\|\|f_t-f_H\|]\\
     & \leq \big(\ebb\big[\|\nabla\ecal(f_t)\|^2\big]\big)^{\frac{1}{2}}\big(\ebb\big[\|f_t-f_H\|^2\big]\big)^{\frac{1}{2}}.
  \end{align*}
  The above inequality together with \eqref{iterate-bound-a} gives
  $$
    \ebb[\big\|\nabla\ecal(f_t)\big\|^2]\geq \frac{\big(\ebb\big[\ecal(f_t)-\ecal(f_H)\big]\big)^2}{\ebb[\|f_t-f_H\|^2]}\geq \frac{\big(\ebb\big[\ecal(f_t)-\ecal(f_H)\big]\big)^2}{\widehat{C}+\gamma\sum_{k=1}^{t}\eta_k^2}.
  $$
  This establishes \eqref{iterate-bound-b} and completes the proof.
\end{proof}

We are now in a position to prove Theorem \ref{thm:suff} for the convergence in expectation. Let $\epsilon>0$ be an arbitrary small number. Our idea is to use Proposition \ref{prop:conv-weak}, based on one-step progress in terms of the distances in RKHSs, to show that $\{A_t\}_{t\in\nbb}$ can be smaller than $\epsilon$ infinitely often. Once $A_{\tilde{t}}\leq\epsilon$ for a sufficiently large $\tilde{t}$, we can use the assumption $\lim_{t\to\infty}\eta_t^\alpha\sum_{k=1}^{t}\eta_k^2=0$ and the one-step progress inequality \eqref{suff-3} in terms of generalization errors to show $A_t\leq\epsilon$ for any $t\geq\tilde{t}$.
\begin{proof}[Proof of Theorem \ref{thm:suff}]
Since $\phi'(y,\cdot)$ is $(\alpha,L)$-H\"older continuous, we can apply the second inequality of \eqref{smooth-hilbert} to show that
$$
  \phi(y,f_{t+1}(x)) \leq \phi(y,f_t(x)) + (f_{t+1}(x)-f_t(x))\phi'(y,f_t(x))+\frac{L}{1+\alpha}|f_{t+1}(x)-f_t(x)|^{1+\alpha}.
$$
According to the reproducing property $f(x)=\inn{f,K_x},\forall f\in H$ and the iteration scheme \eqref{OKL}, we know
\begin{align}
   &\phi(y,f_{t+1}(x)) \leq \phi(y,f_t(x))+\inn{f_{t+1}-f_t,\phi'(y,f_t(x))K_x}+\frac{L}{1+\alpha}|\inn{f_{t+1}-f_t,K_x}|^{1+\alpha} \notag\\
   &\quad \leq \phi(y,f_t(x))-\eta_t\inn{\phi'(y_t,f_t(x_t))K_{x_t},\phi'(y,f_t(x))K_x}+\frac{L\kappa^{1+\alpha}}{1+\alpha}\|f_{t+1}-f_t\|^{1+\alpha}\notag\\
   & \quad\leq \phi(y,f_t(x))-\eta_t\inn{\phi'(y_t,f_t(x_t))K_{x_t},\phi'(y,f_t(x))K_x}+\frac{L\kappa^{2(1+\alpha)}\eta_t^{1+\alpha}}{1+\alpha}|\phi'(y_t,f_t(x_t))|^{1+\alpha}.\label{grad-1}
\end{align}
Putting \eqref{grad-bound} with $\beta=\alpha$ back into \eqref{grad-1} followed with a conditional expectation with respect to $z_t$ and $z$ yields
\begin{align*}
   & \ebb_{z_t}[\ecal(f_{t+1})] = \ebb_{z_t,z}[\phi(y,f_{t+1}(x))]
   \leq \ebb_z[\phi(y,f_t(x))]-\eta_t\big\langle \ebb_{z_t}[\phi'(y_t,f_t(x_t))K_{x_t}],\ebb_z[\phi'(y,f_t(x))K_x]\big\rangle\\
   &\qquad+\frac{L\kappa^{2(1+\alpha)}\eta_t^{1+\alpha}}{1+\alpha}\Big[2^\alpha L^{\frac{1}{\alpha}}(1+\alpha)(\ecal(f_t)-\ecal(f_H))+2^\alpha(1-\alpha)+2^\alpha\ebb_z[|\phi'(y,f_H(x))|^{1+\alpha}]\Big]\ \\
   &\leq \ecal(f_t)\!-\!\eta_t\|\nabla\ecal(f_t)\|^2\!+\!\frac{L\kappa^{2(1+\alpha)}2^\alpha\eta_t^{1+\alpha}\Big[ L^{\frac{1}{\alpha}}(1\!+\!\alpha)(\ecal(f_t)\!-\!\ecal(f_H))\!+\!(1\!-\!\alpha)\!+\!\ebb_z[|\phi'(y,f_H(x))|^{1+\alpha}]\Big]}{1\!+\!\alpha}.
\end{align*}
Subtracting $\ecal(f_H)$ from both sides of the above inequality gives
\begin{multline}
  \ebb_{z_t}[\ecal(f_{t+1})]-\ecal(f_H)\leq \Big[1+L^{1+\frac{1}{\alpha}}\kappa^{2(1+\alpha)}2^\alpha\eta_t^{1+\alpha}\Big](\ecal(f_t)-\ecal(f_H))-\eta_t\|\nabla\ecal(f_t)\|^2\\
  +\frac{L\kappa^{2(1+\alpha)}2^\alpha\eta_t^{1+\alpha}}{1+\alpha}\Big[(1-\alpha)+\ebb_z[|\phi'(y,f_H(x))|^{1+\alpha}]\Big].\label{grad-4}
\end{multline}
Taking expectations over both sides, the above inequality can be written as
\begin{equation}\label{suff-3}
  A_{t+1}\leq (1+a\eta_t^{1+\alpha})A_t+b\eta_t^{1+\alpha}-\eta_t\ebb[\|\nabla\ecal(f_t)\|^2],
\end{equation}
where we introduce the notations
\begin{equation}\label{grad-5}
  a=L^{1+\frac{1}{\alpha}}\kappa^{2(1+\alpha)}2^\alpha\quad\text{and}\quad b=\frac{L\kappa^{2(1+\alpha)}2^\alpha}{1+\alpha}\Big[(1-\alpha)+\ebb_z[|\phi'(y,f_H(x))|^{1+\alpha}]\Big].
\end{equation}
Plugging \eqref{iterate-bound-b} into the above inequality gives
\begin{equation}\label{suff-1}
  A_{t+1}\leq (1+a\eta_t^{1+\alpha})A_t+b\eta_t^{1+\alpha}-\frac{\eta_tA_t^2}{\widehat{C}+\gamma\sum_{k=1}^{t}\eta_k^2},
\end{equation}
where $\widehat{C}$ and $\gamma$ are defined in the proof of Lemma \ref{lem:iterate-bound}.
The assumption $\lim_{t\to\infty}\eta_t^\alpha\sum_{k=1}^{t}\eta_k^2=0$ implies $\lim_{t\to\infty}\eta_t=0$ and therefore the assumption of Proposition \ref{prop:conv-weak} hold.
Let $\epsilon\in(0,1)$ be an arbitrary number.
According to $\liminf_{t\to\infty}A_t=0$ established in Proposition \ref{prop:conv-weak},
we can find a $\tilde{t}\in\nbb$ ($\tilde{t}$ can be sufficiently large) such that
$A_{\tilde{t}}\leq\epsilon$ and
\begin{equation}\label{suff-2}
  \eta_t^\alpha\big(\widehat{C}+\gamma\sum_{k=1}^{t}\eta_k^2\big)\leq \frac{\epsilon^2}{4(a+b)},\quad \eta_t^{1+\alpha}\leq \frac{\epsilon}{2(a+b)}\quad\forall t\geq\tilde{t}.
\end{equation}
We now prove by induction that $A_t\leq \epsilon$ for all $t\geq\tilde{t}$. It suffices to show that $A_{t+1}\leq\epsilon$ under the assumption $A_t\leq\epsilon$ and $t\geq\tilde{t}$. Since $A_t\leq1$, we derive from \eqref{suff-1} that
$$
  A_{t+1}\leq A_t + (a+b)\eta_t^{1+\alpha} - \frac{\eta_tA_t^2}{\widehat{C}+\gamma\sum_{k=1}^t\eta_k^2}.
$$
We now consider two cases.
If $A_t^2\geq (a+b)\eta_t^\alpha\big(\widehat{C}+\gamma\sum_{k=1}^{t}\eta_k^2\big)$, then we know $A_{t+1}\leq A_t \leq \epsilon$. Otherwise, we derive from \eqref{suff-2} that
$$
  A_{t+1} \leq A_t + (a+b)\eta_t^{1+\alpha}\leq \sqrt{(a+b)\eta_t^\alpha\big(\widehat{C}+\gamma\sum_{k=1}^{t}\eta_k^2\big)}+(a+b)\eta_t^{1+\alpha}\leq\epsilon.
$$
Putting the above two cases together we derive $A_{t+1}\leq\epsilon$. That is, $A_t\leq\epsilon$ for all $t\geq\tilde{t}$. Since $\epsilon\in(0,1)$ is arbitrarily chosen, we get $\lim_{t\to\infty}A_t=0$.
\end{proof}

The necessary condition in Theorem \ref{thm:nece} is established by applying the co-coercivity given in Lemma \ref{lem:smooth-hilbert} to bound $\ecal(f_{t+1})$ in terms of $\ecal(f_t)$ from below.
\begin{proof}[Proof of Theorem \ref{thm:nece}]
Since $\phi'(y,\cdot)$ is $(1,L)$-H\"older continuous for any $y\in\ycal$, we have
\begin{align}
  \|\nabla\ecal(f)-\nabla\ecal(\tilde{f})\| & = \big\|\ebb[\phi'(y,f(x))K_x-\phi'(y,\tilde{f}(x))K_x]\big\|
   \leq \ebb\big[|\phi'(y,f(x))-\phi'(y,\tilde{f}(x))|\|K_x\|\big]\notag\\
   & \leq L\ebb\big[|\inn{f-\tilde{f},K_x}|\|K_x\|\big]\leq L\kappa^2\|f-\tilde{f}\|.\label{nece-3}
\end{align}
That is, $\nabla\ecal$ is $(1,L\kappa^2)$-H\"older continuous. Lemma \ref{lem:smooth-hilbert} with $\alpha=1$ and $\nabla\ecal(f_H)=0$ then yield the following inequality
\begin{equation}
  \ecal(f_t) \geq \ecal(f_H)+\inn{f_t-f_H,\nabla\ecal(f_H)}+\frac{\|\nabla\ecal(f_t)-\nabla\ecal(f_H)\|^2}{2L\kappa^2}
   = \ecal(f_H) + \frac{\|\nabla\ecal(f_t)\|^2}{2L\kappa^2}.\label{nece-1}
\end{equation}
It follows from the convexity of $\ecal$ and \eqref{OKL} that
\begin{align*}
  \ecal(f_{t+1}) & \geq \ecal(f_t) + \inn{\nabla\ecal(f_t),f_{t+1}-f_t}=
  \ecal(f_t) - \eta_t\big\langle \nabla\ecal(f_t),\phi'(y_t,f_t(x_t))K_{x_t}\big\rangle.
\end{align*}
Taking expectations over both sides and using \eqref{nece-1}, we derive the following inequality for all $t\in\nbb$
\begin{align}
  \ebb[\ecal(f_{t+1})]
  \geq \ebb[\ecal(f_t)] - \eta_t\ebb[\|\nabla\ecal(f_t)\|^2]\label{nece-2}
  \geq \ebb[\ecal(f_t)]-2L\kappa^2\eta_t\ebb[\ecal(f_t)-\ecal(f_H)].\notag
\end{align}
Hence,
$$
  A_{t+1}\geq \big(1-2L\kappa^2\eta_t\big)A_t,\quad\forall t\in\nbb.
$$
The assumption $\eta_t\leq 1/(6L\kappa^2)$ and the elementary inequality $1-\eta\geq\exp(-2\eta),\forall \eta\in(0,1/3)$ \citep{lin2015learning} then show
$$
  A_{t+1} \geq \exp\big(-4L\kappa^2\eta_t\big)A_t\geq\prod_{k=1}^{t}\exp\big(-4L\kappa^2\eta_k\big)A_1  = \exp\Big(-4L\kappa^2\sum_{k=1}^{t}\eta_k\Big)A_1,
$$
which, together with the condition $\lim_{t\to\infty}A_t=0$ and $A_1\neq0$, then establishes the necessary condition $\sum_{t=1}^{\infty}\eta_t=\infty$.
\end{proof}

\subsection{Proofs for almost sure convergence}\label{sec:proof-conv-ae}
We use the following Doob's martingale convergence theorem \citep[see, e.g.,][page 195]{doob1994graduate} to prove Theorem \ref{thm:conv-ae} on almost sure convergence.
Specifically, we will use the one-step progress inequality in terms of generalization errors to construct a supermartingale, whose almost sure convergence would imply the almost sure convergence of $\{\hat{A}_t\}_{t\in\nbb}$.
\begin{lemma}\label{lem:super-martingale}
  Let $\{\tilde{X}_t\}_{t\in\nbb}$ be a sequence of non-negative random variables and let $\{\fcal_t\}_{t\in\nbb}$ be a nested sequence of sets of random variables with $\fcal_t\subset\fcal_{t+1}$ for all $t\in\nbb$. If $\ebb[\tilde{X}_{t+1}|\fcal_t]\leq \tilde{X}_t$ for every $t\in\nbb$, then $\tilde{X}_t$ converges to a nonnegative random variable $\tilde{X}$ almost surely. Furthermore, $\tilde{X}<\infty$ almost surely.
\end{lemma}

\begin{proof}[Proof of Theorem \ref{thm:conv-ae}]
Eq. \eqref{grad-4} gives
\begin{equation}\label{conv-ae-1}
  \ebb_{z_t}[\hat{A}_{t+1}]\leq (1+a\eta_t^{1+\alpha})\hat{A}_t+b\eta_t^{1+\alpha},\quad\forall t\in\nbb,
\end{equation}
with $a$ and $b$ are defined in the proof of Theorem \ref{thm:suff}.
Denote $c=\prod_{k=1}^{\infty}(1+a\eta_k^{1+\alpha})$, which, according to the elementary inequality $1+\tau\leq\exp(\tau),\tau\geq0$ and \eqref{step-ae}, satisfies
$$
  c\leq\prod_{k=1}^\infty\exp(a\eta_k^{1+\alpha})=\exp\big(a\sum_{k=1}^{\infty}\eta_k^{1+\alpha}\big)<\infty.
$$
Multiplying both sides of \eqref{conv-ae-1} by $\prod_{k=t+1}^{\infty}(1+a\eta_k^{1+\alpha})$, we derive
\begin{align}
  \prod_{k=t+1}^{\infty}(1+a\eta_k^{1+\alpha})\ebb_{z_t}\big[\hat{A}_{t+1}\big] & \leq \prod_{k=t}^{\infty}(1+a\eta_k^{1+\alpha})\hat{A}_t+b\eta_t^{1+\alpha}\prod_{k=t+1}^{\infty}(1+a\eta_k^{1+\alpha}) \notag \\
   & \leq \prod_{k=t}^{\infty}(1+a\eta_k^{1+\alpha})\hat{A}_t+bc\eta_t^{1+\alpha}.\label{conv-ae-2}
\end{align}
Introduce the stochastic process
\begin{equation}\label{conv-ae-3}
  \hat{X}_t=\prod_{k=t}^{\infty}(1+a\eta_k^{1+\alpha})\hat{A}_t+bc\sum_{k=t}^{\infty}\eta_k^{1+\alpha},\quad t\in\nbb
\end{equation}
Eq. \eqref{conv-ae-2} implies
$\ebb_{z_t}[\hat{X}_{t+1}]\leq\hat{X}_t$
for all $t\in\nbb$, that is, $\{\hat{X}_t\}_{t\in\nbb}$ is a supermartingale taking non-negative values. Lemma \ref{lem:super-martingale} then implies that $\lim_{t\to\infty}\hat{X}_t=\hat{X}$ for a non-negative random variable $\hat{X}$ almost surely. Let $\Omega=\{\omega=\{z_t\}_{t\in\nbb}\}$ be the set for which $\{\hat{X}_t(\omega)\}_t$ converges to $\hat{X}(\omega)$ as $t\to\infty$ and $\hat{X}(\omega)<\infty$. Then, $\text{Pr}\{\Omega\}=1$, where $\text{Pr}\{\Omega\}$ denotes the probability with which the event $\Omega$ happens. Let $\omega\in\Omega$ and $\epsilon>0$. Since $\sum_{t=1}^{\infty}\eta_t^{1+\alpha}<\infty$, we can find $\tilde{t}\in\nbb$ such that
$$
  \sum_{t=\tilde{t}}^{\infty}\eta_t^{1+\alpha}<\frac{\epsilon}{3bc},\quad\prod_{k=\tilde{t}}^{\infty}(1+a\eta_k^{1+\alpha})<1+\frac{\epsilon}{3\hat{X}(\omega)+\epsilon}
  \;\;\text{and}\;\;|\hat{X}_t(\omega)-\hat{X}(\omega)|<\frac{\epsilon}{3},\quad\forall t\geq \tilde{t}.
$$
It then follows from \eqref{conv-ae-3} that
\begin{align*}
  \hat{A}_t(\omega)&\leq\hat{X}_t(\omega)\leq \Big(1+\frac{\epsilon}{3\hat{X}(\omega)+\epsilon}\Big)\hat{A}_t(\omega)+\frac{\epsilon}{3}\leq\hat{A}_t(\omega)+\frac{\epsilon\hat{X}_t(\epsilon)}{3\hat{X}(\omega)+\epsilon}+\frac{\epsilon}{3}\\
  &\leq\hat{A}_t(\omega)+\frac{\epsilon\big(\hat{X}(\omega)+\frac{\epsilon}{3}\big)}{3\hat{X}(\epsilon)+\epsilon}+\frac{\epsilon}{3}
  \leq \hat{A}_t(\omega)+\frac{2\epsilon}{3},\quad\forall t\geq \tilde{t},
\end{align*}
from which we derive
$$
  \hat{X}(\omega)-\epsilon\leq\hat{X}_t(\omega)-\frac{2\epsilon}{3}\leq\hat{A}_t(\omega)\leq \hat{X}(\omega)+\epsilon,\quad\forall t\geq \tilde{t}.
$$
That is, $\lim_{t\to\infty}\hat{A}_t(\omega)=\hat{X}(\omega)$ for any $\omega\in\Omega$, i.e., $\lim_{t\to\infty}\hat{A}_t=\hat{X}$ almost surely. Since $\sum_{t=1}^{\infty}\eta_t^{1+\alpha}<\infty$, we know $\sum_{t=1}^{\infty}\eta_t^2<\infty$ and $\lim_{t\to\infty}\eta_t=0$. This further implies $$\lim_{t\to\infty}\eta_t^\alpha\sum_{k=1}^{t}\eta_k^2=0$$ and therefore the assumptions in Theorem \ref{thm:suff} hold. Theorem \ref{thm:suff} shows that
$\lim_{t\to\infty}\ebb[\hat{A}_t]=0$.
By Fatou's lemma, we get
$$
  0\leq\ebb[\hat{X}]=\ebb\big[\lim_{t\to\infty}\hat{A}_t\big]\leq\liminf_{t\to\infty}\ebb[\hat{A}_t]=0,
$$
which implies that $\ebb[\hat{X}]=0$ and therefore $\hat{X}=0$ almost surely since  $\hat{X}$ is non-negative. Combining the above deductions together, we know that $\lim_{t\to\infty}\hat{A}_t=0$ almost surely.
\end{proof}
Our proof of Theorem \ref{thm:ae-borel} is based on the following lemma which can be found in \citet{lin2015learning} as an easy consequence of the Borel-Cantelli Lemma.
\begin{lemma}\label{lem:borel}
  Let $\{\xi_t\}_{t\in\nbb}$ be a sequence of non-negative random variables and $\{\epsilon_t\}_{t\in\nbb}$ be a sequence of positive numbers satisfying $\lim_{t\to\infty}\epsilon_t=0$. If $\sum_{t=1}^{\infty}\mathrm{Pr}\{\xi_t>\epsilon_t\}<\infty$, then $\xi_t$ converges to $0$ almost surely.
\end{lemma}
\begin{proof}[Proof of Theorem \ref{thm:ae-borel}]
  Introduce $\delta_t=t^{-2}$ for all $t\in\nbb$.
  According to Corollary \ref{cor:rate-last}, there exists a constant $\widetilde{C}_1$ such that
  $$
    \text{Pr}\Big\{
    t^{\min\{1-\theta,(\alpha+1)\theta-1\}-\epsilon}\hat{A}_t\geq\widetilde{C}_1t^{-\epsilon}\log^2\frac{t}{\delta_t}\Big\}\leq\delta_t.
  $$
  Since $\sum_{t=1}^\infty \delta_t<\infty$ and $\lim_{t\to\infty}t^{-\epsilon}\log^2\frac{t}{\delta_t}=0$, we can apply Lemma \ref{lem:borel} here to show \eqref{ae-borel}. The proof is complete.
\end{proof}

\subsection{Proofs for convergence rates with high probability}\label{sec:proof-conv-prob}
Our discussion on high-probability convergence rates roots its foundation on the following concentration inequalities of martingales. Part (a) is the Azuma-Hoeffding inequality for martingales with bounded differences~\citep{hoeffding1963probability}, while Part (b) is a Bernstein-type inequality which exploits information on variances to derive improved concentration inequalities for martingales~\citep{zhang2005data}. A remarkable property of this Bernstein-type inequality is that it involves a conditional variance which itself is a random variable.
\begin{lemma}\label{lem:martingale}
  Let $z_1,\ldots,z_n$ be a sequence of random variables such that $z_k$ may depend on the previous random variables $z_1,\ldots,z_{k-1}$ for all $k=1,\ldots,n$. Consider a sequence of functionals $\xi_k(z_1,\ldots,z_k),k=1,\ldots,n$.
  \begin{enumerate}[(a)]
    \item Assume that $|\xi_k-\ebb_{z_k}[\xi_k]|\leq b_k$ for each $k$. Let $\delta\in(0,1)$. With probability at least $1-\delta$ we have
    \begin{equation}\label{hoeffding}
      \sum_{k=1}^{n}\xi_k-\sum_{k=1}^{n}\ebb_{z_k}[\xi_k]\leq \Big(2\sum_{k=1}^{n}b_k^2\log\frac{1}{\delta}\Big)^{\frac{1}{2}}.
    \end{equation}
    \item Assume that $\xi_k-\ebb_{z_k}[\xi_k]\leq b$ for each $k$. Let $\rho>0$ and $\delta\in(0,1)$. With probability at least $1-\delta$ we have
    \begin{equation}\label{bernstein}
      \sum_{k=1}^{n}\xi_k-\sum_{k=1}^{n}\ebb_{z_k}[\xi_k]\leq \frac{(e^\rho-\rho-1)\sigma_n^2}{\rho b}+\frac{b\log\frac{1}{\delta}}{\rho},
    \end{equation}
  where $\sigma_n^2=\sum_{k=1}^{n}\ebb_{z_k}(\xi_k-\ebb_{z_k}\xi_k)^2$ is the conditional variance.
  \end{enumerate}
\end{lemma}

Since $\phi'(y,\cdot)$ is $(\alpha,L)$-H\"older continuous, convex and non-negative, Proposition 1 in \citet{ying2017unregularized} shows that $\phi(y,\cdot)$ satisfies the following self-bounding property
$$
  |\phi'(y,s)|^{\frac{1+\alpha}{\alpha}}\leq \frac{(1+\alpha)^{1+\frac{1}{\alpha}}}{\alpha}L^{\frac{1}{\alpha}}\phi(y,s),\quad \forall y\in\ycal,s\in\rbb.
$$
The Young's inequality \eqref{young} then implies
\begin{align}
  & |\phi'(y,s)|^2\leq \alpha^{-\frac{2\alpha}{1+\alpha}}(1+\alpha)^2L^{\frac{2}{1+\alpha}}\phi(y,s)^{\frac{2\alpha}{1+\alpha}}\notag\\
  & \leq\alpha^{-\frac{2\alpha}{1+\alpha}}L^{\frac{2}{1+\alpha}}(1+\alpha)
  \Big(2\alpha\phi(y,s)+1-\alpha\Big)
  =A\phi(y,s)+B,\label{self-bound}
\end{align}
where
\begin{equation}\label{A-B}
  A=2\alpha^{\frac{1-\alpha}{1+\alpha}}L^{\frac{2}{1+\alpha}}(1+\alpha)\quad\text{and}\quad B=\alpha^{-\frac{2\alpha}{1+\alpha}}L^{\frac{2}{1+\alpha}}(1-\alpha^2).
\end{equation}

Below we will use Part (b) of Lemma \ref{lem:martingale} to show almost boundedness of $\{f_t\}_{t\in\nbb}$ with high probability (Proposition \ref{prop:iterate-bound-pr}). To this aim, we first establish a crude bound on the iterates $\{f_t\}_{t\in\nbb}$ in terms of the step size sequence.

\begin{lemma}\label{lem:iteration-bound-ae}
  Let $\{f_t\}_{t\in\nbb}$ be the sequence given by \eqref{OKL}.
  Assume $\eta_t\leq\frac{1}{A\kappa^2}$ for all $t\in\nbb$.
  Then, the following inequalities hold for all $t\in\nbb$
  \begin{equation}\label{iteration-bound-ae-a}
    \|f_{t+1}-f_H\|^2 \leq C_1\sum_{k=0}^{t}\eta_k\quad\text{and}\quad\|f_{t+1}\|^2\leq C_1\sum_{k=1}^{t}\eta_k,
  \end{equation}
  where we introduce for brevity $\eta_0=1$ and
  \begin{equation}\label{C-1}
     C_1=\|f_H\|_2^2+A^{-1}B+2\max\big\{\sup_{y\in\ycal}\phi(y,0),\sup_{z\in\zcal}\phi(y,f_H(x))\big\}.
  \end{equation}
  Furthermore, if $\eta_t\leq\frac{1}{A\kappa^2}$ and $\eta_{t+1}\leq\eta_t$ for all $t\in\nbb$, we have
  \begin{equation}\label{iteration-bound-ae-b}
    \sum_{k=1}^{t}\eta_k^2\phi(y_k,f_k(x_k))\leq \eta_1\|f_H\|^2+C_2\sum_{k=1}^{t}\eta_k^2,
  \end{equation}
  where we introduce
  \begin{equation}\label{C-2}
    C_2=2\sup_{z\in\zcal}\phi(y,f_H(x))+\eta_1\kappa^2B.
  \end{equation}
\end{lemma}
\begin{proof}
Plugging \eqref{self-bound} into \eqref{conv-weak-2} gives
\begin{align}
  \|f_{t+1}&-f_H\|^2 \leq \|f_t-f_H\|^2+\eta_t^2\kappa^2[A\phi(y_t,f_t(x_t))+B]+2\eta_t[\phi(y_t,f_H(x_t))-\phi(y_t,f_t(x_t))] \notag\\
   & = \|f_t-f_H\|^2+2\eta_t\phi(y_t,f_H(x_t))+\eta_t^2\kappa^2 B+\eta_t(A\eta_t\kappa^2-2)\phi(y_t,f_t(x_t))\label{iteration-bound-ae-1}\\
   & \leq \|f_t-f_H\|^2+2\eta_t\phi(y_t,f_H(x_t))+\eta_t^2\kappa^2 B\leq\|f_t-f_H\|^2+\eta_t\big(2\phi(y_t,f_H(x_t))+A^{-1}B\big),\notag
\end{align}
where the last two inequalities follow from the assumption $\eta_t\leq\frac{1}{A\kappa^2}$.
According to the definitions of $C_1$ in \eqref{C-1} and $\eta_0$, it then follows that
$$
  \|f_{t+1}-f_H\|^2 = \|f_H\|^2 + \sum_{k=1}^{t}\big[\|f_{k+1}-f_H\|^2-\|f_k-f_H\|^2\big]\leq C_1\sum_{k=0}^{t}\eta_k.
$$
This establishes the first inequality in \eqref{iteration-bound-ae-a}. We now prove the second inequality in \eqref{iteration-bound-ae-a}.
Notice that \eqref{conv-weak-2} also holds if we replace $f_H$ with $0$. This, together with \eqref{self-bound} and $\eta_t\leq\frac{1}{A\kappa^2}$, gives
\begin{align*}
  \|f_{t+1}\|^2 & \leq \|f_t\|^2+\eta_t^2\kappa^2[A\phi(y_t,f_t(x_t))+B]+2\eta_t[\phi(y_t,0)-\phi(y_t,f_t(x_t))] \\
   & = \|f_t\|^2+2\eta_t\phi(y_t,0)+\eta_t^2\kappa^2 B+\eta_t(A\eta_t\kappa^2-2)\phi(y_t,f_t(x_t))\\
   & \leq \|f_t\|^2+2\eta_t\phi(y_t,0)+\eta_tA^{-1}B.
\end{align*}
It is now clear
$$
\|f_{t+1}\|^2=\sum_{k=1}^{t}\big[\|f_{k+1}\|^2-\|f_k\|^2\big]\leq C_1\sum_{k=1}^{t}\eta_k.
$$

We now show \eqref{iteration-bound-ae-b}.
Applying $\eta_t\leq\frac{1}{A\kappa^2}$ in \eqref{iteration-bound-ae-1} gives
\begin{equation}\label{iterate-bound-ae-1}
  \eta_t\phi(y_t,f_t(x_t))\leq \|f_t-f_H\|^2-\|f_{t+1}-f_H\|^2+2\eta_t\phi(y_t,f_H(x_t))+\eta_t^2\kappa^2 B.
\end{equation}
Multiplying both sides of the above inequality by $\eta_t$ and using $\eta_{t+1}\leq\eta_t$, we derive
\begin{align*}
  \eta_t^2\phi(y_t,f_t(x_t)) & \leq \eta_t\big[\|f_t-f_H\|^2-\|f_{t+1}-f_H\|^2\big]+2\eta_t^2\phi(y_t,f_H(x_t))+\eta_t^3\kappa^2B \\
   & \leq \eta_t\|f_t-f_H\|^2-\eta_{t+1}\|f_{t+1}-f_H\|^2+2\eta_t^2\phi(y_t,f_H(x_t))+\eta_t^3\kappa^2B.
\end{align*}
Taking a summation of the above inequality gives \eqref{iteration-bound-ae-b}.
The proof is complete.
\end{proof}


Based on the above lemma, Proposition \ref{prop:iterate-prob} gives a high-probability bound on $\|f_{t+1}-f_H\|^2$ in terms of $\sum_{k=1}^{t}\eta_k^2\|f_k-f_H\|^2$. Proposition \ref{prop:iterate-prob} is proved based on a one-step progress inequality \eqref{iterate-prob-0} in terms of the RKHS distances, where the involved martingale is controlled by a Bernstein-type inequality with the dominant variance term cancelled out by the negative term $-2\sum_{k=1}^{t}\eta_kA_k$ existing in the one-step progress inequality.
\begin{proposition}\label{prop:iterate-prob}
  Suppose assumptions in Theorem \ref{thm:conv-prob} hold.
  Let $\delta\in(0,1)$ and $C_\eta,C_3,C_4$ be constants defined by
  \begin{gather}
     C_\eta=\sup_{k\in\nbb}\eta_k\sum_{j=0}^{k}\eta_j<\infty,\label{C-eta}\\
    C_3=\sup_{z_k\in\zcal}\big\|\phi'(y_k,f_H(x_k))K_{x_k}-\ebb_z[\phi'(y,f_H(x))K_x]\big\|,\ C_4=\frac{2(1-\alpha)\kappa^2}{1+\alpha}+2\kappa^2\ebb_z\big[|\phi'(y,f_H(x))|^2\big].\label{C-3}
  \end{gather}
  Then, there exists a constant $\rho_1$ (explicitly given in the proof and independent of $t$ as well as the step size sequence) such that the following inequality holds with probability at least $1-\delta$
  \begin{multline}\label{iterate-prob}
    \|f_{t+1}-f_H\|^2
   \leq (\eta_1\kappa^2A+1)\|f_H\|^2 +(AC_2+B)\kappa^2\sum_{k=1}^{t}\eta_k^2 + \frac{C_4\sum_{k=1}^{t}\big[\eta_k^2\|f_H-f_k\|^2\big]}{2C_1C_\eta\kappa^2L^{\frac{1}{\alpha}}}\\
   +\frac{\big(2C_3C_1^{\frac{1}{2}}C_\eta+4L(C_1^{\frac{1}{2}}\kappa)^{\alpha+1}C_\eta\big)\log\frac{1}{\delta}}{\rho_1}.
  \end{multline}
\end{proposition}
\begin{proof}
The assumption $\sum_{t=1}^{\infty}\eta_t^2<\infty$ implies that $C_\eta$ in \eqref{C-eta} is well defined since $\eta_k\sum_{j=1}^{k}\eta_j\leq\sum_{j=1}^{k}\eta_j^2<\infty$.
According to \eqref{conv-weak-1} and \eqref{self-bound}, we derive
\begin{multline}\label{iterate-prob-0}
  \|f_{k+1}-f_H\|^2 \leq \|f_k-f_H\|^2+\eta_k^2\kappa^2\big(A\phi(y_k,f_k(x_k))+B\big) +  2\eta_k\langle f_H-f_k,\ebb_{z_k}[\phi'(y_k,f_k(x_k))K_{x_k}]\big\rangle\\
  +2\eta_k\big\langle f_H-f_k,\phi'(y_k,f_k(x_k))K_{x_k}-\ebb_{z_k}[\phi'(y_k,f_k(x_k))K_{x_k}]\big\rangle.
\end{multline}
Using the convexity of $\phi$ followed with a summation from $k=1$ to $t$ gives
\begin{align}
   \|f_{t+1}-f_H\|^2 & \leq \|f_H\|^2 + \kappa^2\sum_{k=1}^{t}\eta_k^2\big(A\phi(y_k,f_k(x_k))+B\big) +  2\sum_{k=1}^{t}\eta_k\big[\ecal(f_H)-\ecal(f_k)\big]\notag\\
   &\qquad +2\sum_{k=1}^{t}\eta_k\big\langle f_H-f_k,\phi'(y_k,f_k(x_k))K_{x_k}-\ebb_{z_k}\big[\phi'(y_k,f_k(x_k))K_{x_k}\big]\big\rangle\notag\\
   & \leq (\eta_1A\kappa^2+1)\|f_H\|^2 +(AC_2+B)\kappa^2\sum_{k=1}^{t}\eta_k^2 +  2\sum_{k=1}^{t}\eta_k\big[\ecal(f_H)-\ecal(f_k)\big]\notag\\
   &\qquad +2\sum_{k=1}^{t}\eta_k\big\langle f_H-f_k,\phi'(y_k,f_k(x_k))K_{x_k}-\ebb_{z_k}\big[\phi'(y_k,f_k(x_k))K_{x_k}\big]\big\rangle,\label{iterate-prob-1}
\end{align}
where the last inequality is due to \eqref{iteration-bound-ae-b}.
We now estimate the last term of the above inequality with Lemma \ref{lem:martingale}. To this aim, we need to control both the magnitudes and variances for the martingale difference sequences.

Introduce a sequence of functionals $\xi_k,k\in\nbb$ as follows
$$
  \xi_k=\eta_k\big\langle f_H-f_k,\phi'(y_k,f_k(x_k))K_{x_k}-\ebb_{z_k}\big[\phi'(y_k,f_k(x_k))K_{x_k}\big]\big\rangle.
$$
It is clear
\begin{align*}
   & \big\|\phi'(y_k,f_k(x_k))K_{x_k}-\ebb_{z_k}[\phi'(y_k,f_k(x_k))K_{x_k}]\big\| \leq \big\|\phi'(y_k,f_k(x_k))K_{x_k}-\phi'(y_k,f_H(x_k))K_{x_k}\big\|\\
   &\quad+\big\|\phi'(y_k,f_H(x_k))K_{x_k}-\ebb_{z}[\phi'(y,f_H(x))K_x]\big\|+\ebb_{z}[\|(\phi'(y,f_H(x))-\phi'(y,f_k(x)))K_x\|]\\
   & \leq \sup_{z_k\in\zcal}\big\|\phi'(y_k,f_H(x_k))K_{x_k}-\ebb_{z}[\phi'(y,f_H(x))K_{x}]\big\|+2L\kappa\sup_{x\in\xcal}|f_k(x)-f_H(x)|^\alpha,
\end{align*}
where we have used the Jensen's inequality in the first step.
But
$$
  |f_k(x)-f_H(x)|=|\inn{f_k-f_H,K_{x}}|\leq\|f_k-f_H\|\kappa.
$$
Combining the above two inequalities and using the definition of $C_3$ in \eqref{C-3} give
\begin{equation}\label{iterate-prob-3}
  \big\|\phi'(y_k,f_k(x_k))K_{x_k}-\ebb_{z_k}[\phi'(y_k,f_k(x_k))K_{x_k}]\big\| \leq C_3+2L\|f_k-f_H\|^\alpha\kappa^{\alpha+1}.
\end{equation}
It then follows from \eqref{iteration-bound-ae-a} and $\ebb_{z_k}[\xi_k]=0$ that (note $\eta_0=1$)
\begin{align}
  \xi_k-\ebb_{z_k}[\xi_k]&=\xi_k\leq \eta_k\|f_H-f_k\|\big\|\phi'(y_k,f_k(x_k))K_{x_k}-\ebb_{z_k}\big[\phi'(y_k,f_k(x_k))K_{x_k}\big]\big\| \notag\\
   & \leq\eta_k C_3\|f_H-f_k\|+2L\eta_k\kappa^{\alpha+1}\|f_H-f_k\|^{1+\alpha}\label{iterate-prob-2}\\
   & \leq\eta_k C_3C_1^{\frac{1}{2}}\big(\sum_{j=0}^{k-1}\eta_j\big)^{\frac{1}{2}}+2L(C_1^{\frac{1}{2}}\kappa)^{\alpha+1}\eta_k\big(\sum_{j=0}^{k-1}\eta_j\big)^{\frac{1+\alpha}{2}}\notag\\
   & \leq C_3C_1^{\frac{1}{2}}C_\eta+2L(C_1^{\frac{1}{2}}\kappa)^{\alpha+1}C_\eta.\notag
\end{align}
Here we have used the definition of $C_\eta$ given in \eqref{C-eta}.
Furthermore, according to Lemma \ref{lem:grad-bound} with $\beta=1$ and the definition of $C_4$ in \eqref{C-3}, the conditional variances can be controlled by (note $\ebb[(\xi-\ebb[\xi])^2]<\ebb[\xi^2]$ for a real-valued random variable $\xi$)
\begin{align*}
   \sum_{k=1}^{t}\ebb_{z_k}(\xi_k-\ebb_{z_k}[\xi_k])^2 &\leq \sum_{k=1}^{t}\eta_k^2\ebb_{z_k}\big[\big\langle f_H-f_k,\phi'(y_k,f_k(x_k))K_{x_k}\big\rangle^2\big] \\
   & \leq \sum_{k=1}^{t}\eta_k^2\|f_H-f_k\|^2\kappa^2\ebb_{z_k}[|\phi'(y_k,f_k(x_k))|^2]\\
   & \leq \sum_{k=1}^{t}\eta_k^2\|f_H-f_k\|^2\Big(4\kappa^2 L^{\frac{1}{\alpha}}[\ecal(f_k)-\ecal(f_H)]+C_4\Big).
\end{align*}
According to \eqref{iteration-bound-ae-a} and the definition of $C_\eta$ in \eqref{C-eta}, we can further get
\begin{align*}
   \sum_{k=1}^{t}\ebb_{z_k}(\xi_k-\ebb_{z_k}[\xi_k])^2 & \leq 4L^{\frac{1}{\alpha}}C_1\kappa^2\sum_{k=1}^{t}\Big[\eta_k^2\big(\sum_{j=0}^{k-1}\eta_j\big)\big(\ecal(f_k)-\ecal(f_H)\big)\Big]+C_4\sum_{k=1}^{t}\eta_k^2\|f_H-f_k\|^2\\
   &\leq 4L^{\frac{1}{\alpha}}C_1C_\eta\kappa^2\sum_{k=1}^{t}\eta_k\big(\ecal(f_k)-\ecal(f_H)\big)+C_4\sum_{k=1}^{t}\eta_k^2\|f_H-f_k\|^2.
\end{align*}

Let $\rho_1$ be the largest positive constant such that (such $\rho_1$ exists since $\lim_{\rho\to0}\frac{e^{\rho}-\rho-1}{\rho}=0$)
$$
  \frac{(e^{\rho_1}-\rho_1-1)L^{\frac{1}{\alpha}}C_1^{\frac{1}{2}}\kappa^2}{C_3+2LC_1^{\frac{\alpha}{2}}\kappa^{\alpha+1}}\leq\frac{\rho_1}{4}.
$$
Since $C_1$ and $C_3$ do not depend on the step size sequence, $\rho_1$ is also a constant independent of the step size sequence.
Plugging the above estimates on the magnitudes and variances of $\xi_k$ into Part (b) of Lemma \ref{lem:martingale}, we derive the following inequality with probability at least $1-\delta$
\begin{align*}
  \sum_{k=1}^{t}\xi_k &\leq \frac{(e^\rho_1-\rho_1-1)}{\rho_1 \big( C_3C_1^{\frac{1}{2}}C_\eta+2L(C_1^{\frac{1}{2}}\kappa)^{\alpha+1}C_\eta\big)}\Big[4L^{\frac{1}{\alpha}}C_1C_\eta\kappa^2\sum_{k=1}^{t}\eta_k\big(\ecal(f_k)-\ecal(f_H)\big)\\
  &\qquad\qquad\qquad\qquad\qquad\qquad  +C_4\sum_{k=1}^{t}\eta_k^2\|f_H-f_k\|^2\Big]+\frac{\big(C_3C_1^{\frac{1}{2}}C_\eta+2L(C_1^{\frac{1}{2}}\kappa)^{\alpha+1}C_\eta\big)\log\frac{1}{\delta}}{\rho_1}\\
  &\leq \sum_{k=1}^{t}\eta_k\big(\ecal(f_k)-\ecal(f_H)\big)+\frac{C_4\sum_{k=1}^{t}\big[\eta_k^2\|f_H-f_k\|^2\big]}{4C_1C_\eta\kappa^2L^{\frac{1}{\alpha}}}+\frac{ C_3C_1^{\frac{1}{2}}C_\eta+2L(C_1^{\frac{1}{2}}\kappa)^{\alpha+1}C_\eta\log\frac{1}{\delta}}{\rho_1}.
\end{align*}
Plugging this inequality into \eqref{iterate-prob-1} gives the stated inequality with probability at least $1-\delta$.
\end{proof}

According to Proposition \ref{prop:iterate-bound-pr} and the assumption $\sum_{k=1}^{\infty}\eta_k^2<\infty$, one can show essentially that $\max\limits_{1\leq t\leq T}\|f_t-f_H\|^2\leq\frac{1}{2}\max\limits_{1\leq t\leq T}\|f_t-f_H\|^2+c\log T$ for a constant $c>0$, from which one can establish the boundedness of the iterates with high probability (up to logarithmic factors).

\begin{proof}[Proof of Proposition \ref{prop:iterate-bound-pr}]
  We define the {subset $\Omega\subset\zcal^T$ by}
  $$
    \Omega=\Big\{(z_1,\ldots,z_T):\|f_{t+1}-f_H\|^2\leq C_5+\frac{C_4\sum_{k=1}^{t}\big[\eta_k^2\|f_H-f_k\|^2\big]}{2C_1C_\eta\kappa^2L^{\frac{1}{\alpha}}}+C_6\log\frac{T}{\delta}\;\text{for all }t=1,\ldots,T\Big\},
  $$
  where we introduce
  \begin{equation}\label{C-5}
    C_5=(\eta_1\kappa^2A+1)\|f_H\|^2 +(AC_2+B)\kappa^2\sum_{k=1}^{\infty}\eta_k^2,\quad C_6=\frac{2\kappa C_3C_1^{\frac{1}{2}}C_\eta+4L(C_1^{\frac{1}{2}}\kappa)^{\alpha+1}C_\eta}{\rho_1}.
  \end{equation}
  Applying Proposition \ref{prop:iterate-prob} together with union bounds on probabilities of events, we have $\text{Pr}\{\Omega\}\geq1-\delta$. Since $\sum_{t=1}^{\infty}\eta_t^2<\infty$, there exists a $t_2\in\nbb$ such that
  $$
     C_4\sum_{k=t_2}^{\infty}\eta_k^2\leq C_1C_\eta\kappa^2L^{\frac{1}{\alpha}}.
  $$
  Under the event $\Omega$, we know
  \begin{align*}
    \|f_{t+1}-f_H\|^2 &\leq C_5+\frac{C_4\sum_{k=1}^{t_2}\big[\eta_k^2\|f_H-f_k\|^2\big]}{2C_1C_\eta\kappa^2L^{\frac{1}{\alpha}}}+\frac{C_4\sum_{k=t_2+1}^{t}\big[\eta_k^2\|f_H-f_k\|^2\big]}{2C_1C_\eta\kappa^2L^{\frac{1}{\alpha}}}+C_6\log\frac{T}{\delta}\notag\\
    & \leq
    C_5+C_7+\frac{1}{2}\max_{t_2<k\leq t}\|f_k-f_H\|^2+C_6\log\frac{T}{\delta}\\
    & \leq
    C_5+C_7+\frac{1}{2}\max_{1\leq k\leq T}\|f_k-f_H\|^2+C_6\log\frac{T}{\delta},\qquad\forall t=1,\ldots,T.
  \end{align*}
  where we have used the inequality
  $$
    \frac{C_4\sum_{k=1}^{t_2}\big[\eta_k^2\|f_H-f_k\|^2\big]}{2C_1C_\eta\kappa^2L^{\frac{1}{\alpha}}} \leq \frac{C_4C_1\sum_{k=1}^{t_2}\big[\eta_k^2\sum_{j=0}^{k-1}\eta_j\big]}{2C_1C_\eta\kappa^2L^{\frac{1}{\alpha}}}:=C_7.
  $$
  Under the event $\Omega$, it is now clear that
  $$
    \max_{1\leq t\leq T}\|f_{t}-f_H\|^2\leq C_5+C_7+\frac{1}{2}\max_{1\leq k\leq T}\|f_k-f_H\|^2+C_6\log\frac{T}{\delta}.
  $$
  Solving the above linear inequality yields the stated inequality with $\bar{C}=\max\{2(C_5+C_6+C_7),1\}$ with probability at least $1-\delta$.
\end{proof}

We are now in a position to prove Theorem \ref{thm:conv-prob} on general high-probability convergence rates for a weighted average of iterates. The underlying idea is to construct a modified martingale difference sequence by imposing a constraint on the iterates, which is then estimated by applying the Azuma-Hoeffding inequality on martingales. Furthermore, according to Proposition \ref{prop:iterate-bound-pr}, this modified martingale difference sequence would be identical to the original martingale difference sequence with high probability. Let $\mathbb{I}_{\mathcal{A}}$ denote the indicator function of an event $\mathcal{A}$.
\begin{proof}[Proof of Theorem \ref{thm:conv-prob}]
We now introduce the following sequence of functionals $\xi_k',k=1,\ldots,T$ by
$$
  \xi'_k=\eta_k\big\langle f_H-f_k,\phi'(y_k,f_k(x_k))K_{x_k}-\ebb_{z_k}\big[\phi'(y_k,f_k(x_k))K_{x_k}\big]\big\rangle\mathbb{I}_{\{\|f_k-f_H\|^2\leq \bar{C}\log\frac{2T}{\delta}\}},
$$
where $\bar{C}$ is defined in Proposition \ref{prop:iterate-bound-pr}.
Analogous to \eqref{iterate-prob-2}, we have
\begin{align}
  |\xi_k'| & \leq \Big[\eta_k C_3\|f_H-f_k\|+2L\eta_k\kappa^{\alpha+1}\|f_H-f_k\|^{1+\alpha}\Big]\mathbb{I}_{\{\|f_k-f_H\|^2\leq \bar{C}\log\frac{2T}{\delta}\}} \notag\\
   & \leq \big(C_3+2L\kappa^{\alpha+1}\big)\eta_k\max\Big(\|f_H-f_k\|^2,1\Big)\mathbb{I}_{\{\|f_k-f_H\|^2\leq \bar{C}\log\frac{2T}{\delta}\}}\notag\\
   & \leq \big(C_3+2L\kappa^{\alpha+1}\big)\eta_k\bar{C}\log\frac{2T}{\delta}:=b_k.\label{conv-prob-1}
\end{align}
It is clear that $\ebb_{z_k}[\xi'_k]=0$ and $\xi_k'$ only depends on $z_1,\ldots,z_k$.
According to Part (a) of Lemma \ref{lem:martingale}, there exists a subset $\Omega'=\{(z_1,\ldots,z_T):z_1,\ldots,z_T\in\zcal\}\subset\zcal^T$ with probability measure $\text{Pr}\{\Omega'\}\geq 1-\frac{\delta}{2}$ such that for any $(z_1,\ldots,z_T)\in\Omega'$ the following inequality holds
\begin{align*}
  \sum_{k=1}^{T}\xi_k' \leq \Big(2\sum_{k=1}^{T}b_k^2\log\frac{2}{\delta}\Big)^{\frac{1}{2}}
  \leq \big(C_3+2L\kappa^{\alpha+1}\big)\bar{C}\log\frac{2T}{\delta}\Big(2\log\frac{2}{\delta}\sum_{k=1}^{T}\eta_k^2\Big)^{\frac{1}{2}}.
\end{align*}
According to Proposition \ref{prop:iterate-bound-pr}, there exists a subset $\Omega=\{(z_1,\ldots,z_T):z_1,\ldots,z_T\in\zcal\}\subset\zcal^T$ with probability measure $\text{Pr}\{\Omega\}\geq1-\frac{\delta}{2}$ such that for any $(z_1,\ldots,z_T)\in\Omega$ the following inequality holds
$$
  \max_{1\leq k\leq T}\|f_k-f_H\|^2\leq\bar{C}\log\frac{2T}{\delta}.
$$
Let $\{\xi_k\}_k$ be the martingale difference sequence defined in the proof of Proposition \ref{prop:iterate-prob}.
For any $(z_1,\ldots,z_T)\in\Omega\cap\Omega'$, we then have
\begin{align*}
  \sum_{k=1}^{T}\xi_k & = \sum_{k=1}^{T}\xi_k'\leq \big(C_3+2L\kappa^{\alpha+1}\big)\bar{C}\log\frac{2T}{\delta}\Big(2\log\frac{2}{\delta}\sum_{k=1}^{T}\eta_k^2\Big)^{\frac{1}{2}}.
\end{align*}
Under this intersection of these two events, it follows from \eqref{iterate-prob-1} and the definition of $C_5$ given in \eqref{C-5} show
\begin{align*}
  2\sum_{k=1}^{T}\eta_k\big[\ecal(f_k)-\ecal(f_H)\big] &\leq (\eta_1A\kappa^2+1)\|f_H\|^2 +(AC_2+B)\kappa^2\sum_{k=1}^{T}\eta_k^2 + 2\sum_{k=1}^{T}\xi_k\\
  &\leq C_5+2\big(C_3+2L\kappa^{\alpha+1}\big)\bar{C}\log\frac{2T}{\delta}\Big(2\log\frac{2}{\delta}\sum_{k=1}^{T}\eta_k^2\Big)^{\frac{1}{2}}.
\end{align*}
But $\text{Pr}\{\Omega\cap\Omega'\}\geq1-\delta$. Therefore, the first inequality of \eqref{conv-prob} holds with probability at least $1-\delta$ and
$$
  \widetilde{C}=\frac{C_5}{2}+\big(C_3+2L\kappa^{\alpha+1}\big)\bar{C}\Big(2\sum_{k=1}^\infty\eta_k^2\Big)^{\frac{1}{2}}.
$$
The second inequality of \eqref{conv-prob} follows from the convexity of $\ecal(\cdot)$. The proof is complete.
\end{proof}

Other than the high-probability bounds for the weighted average of iterates $\bar{f}_T^\eta$, we can also derive similar results for the uniform average of iterates $\bar{f}_T$. If we choose the step sizes $\eta_t=\eta_1(t\log^\beta t)^{-\frac{1}{2}}$ with $\beta>1$, then Proposition \ref{prop:unif-average} implies $\ecal(\bar{f}_T)-\ecal(f_H)=O(T^{-\frac{1}{2}}\log^{\frac{3}{2}}\frac{T}{\delta})$ with probability at least $1-\delta$.
We present the proof in the appendix due to its similarity to the proof of Theorem \ref{thm:conv-prob}.
\begin{proposition}\label{prop:unif-average}
  Suppose assumptions in Theorem \ref{thm:conv-prob} hold. Then, for any $\delta\in(0,1)$, with probability at least $1-\frac{\delta}{2}$ we have
  $$
    \sum_{t=1}^{T}[\ecal(f_t)-\ecal(f_H)]\!=\!O\big(\big(T^{\frac{1}{2}}\!+\!\sum_{t=1}^{T}\eta_t\big)\log^{\frac{3}{2}}\frac{2T}{\delta}\big)\;\;\text{and}\;\;
    \ecal(\bar{f}_T)-\ecal(f_H)\!=\!O\big(\big(T^{-\frac{1}{2}}\!+\!T^{-1}\sum_{t=1}^{T}\eta_t\big)\log^{\frac{3}{2}}\frac{2T}{\delta}\big).
  $$
\end{proposition}

Theorem \ref{thm:last} is a specific case of Proposition \ref{prop:last-general} with $\widetilde{T}=\half$. The step-stone in proving this proposition is the inequality \eqref{last-2} following from the one-step progress \eqref{last-one-step-progress} in terms of generalization errors. The first term on the right hand side of \eqref{last-2} can be tackled by Theorem \ref{thm:conv-prob} on a weighted summation of $\hat{A}_t$ deduced from the one-step analysis in terms of RKHS distances. The variance of the martingales $\sum_{t=\tilde{t}}^{T}\bar{\xi}_t$ can be controlled by $\sum_{t=\tilde{t}}^{T}\eta_t\|\nabla\ecal(f_t)\|^2$, which is then cancelled out by the third term $-\sum_{t=\tilde{t}}^{T}\eta_t\|\nabla\ecal(f_t)\|^2$. A notable fact is that the martingale difference $\bar{\xi}_t-\ebb_{z_t}[\bar{\xi}_t]$ is bounded by $O(\eta_{\widetilde{T}})$ for all $t\geq\widetilde{T}$ with high probability, which would be small if $\widetilde{T}$ is large. We can balance the three terms on the right hand side of \eqref{last-general} by choosing an appropriate $\widetilde{T}$.

\begin{proposition}\label{prop:last-general}
Suppose that the assumptions in Theorem \ref{thm:conv-prob} hold. Let $\widetilde{T}\in\nbb$ satisfy $1\leq \widetilde{T}\leq T$. Then, there exists a constant $\widetilde{C}'$ independent of $T$ and $\widetilde{T}$ (explicitly given in the proof) such that for any $\delta\in(0,1)$ the following inequality holds with probability at least $1-\delta$
  \begin{equation}\label{last-general}
    \ecal(f_{T+1})-\ecal(f_H)\leq\widetilde{C}'\max\Big\{\big[\sum_{t=\widetilde{T}}^{T}\eta_t\big]^{-1},\eta_{\widetilde{T}}^\alpha,\sum_{t=\widetilde{T}}^{T}\eta_t^{1+\alpha}\Big\}\log^2\frac{3T}{\delta}.
  \end{equation}
\end{proposition}
\begin{proof}
Recall that $\hat{A}_t=\ecal(f_t)-\ecal(f_H)$.
According to the proof of Theorem \ref{thm:conv-prob}, there exists a subset $\Omega=\{(z_1,\ldots,z_T):z_1,\ldots,z_T\in\zcal\}\subset\zcal^T$ with $\text{Pr}\{\Omega\}\geq1-\frac{2\delta}{3}$ such that for any $(z_1,\ldots,z_T)\in\Omega$, we have
\begin{equation}\label{last-0}
  \sum_{t=1}^{T}\eta_t\hat{A}_t\leq\widetilde{C}\log^{\frac{3}{2}}\frac{3T}{\delta}\quad\text{and}\quad\max_{1\leq t\leq T}\|f_t-f_H\|^2\leq\bar{C}\log\frac{3T}{\delta},
\end{equation}
where $\widetilde{C}$ and $\bar{C}$ are constants independent of $T$ and $\delta$.
Under the event of $\Omega$, we have $\sum_{t=\widetilde{T}}^{T}\eta_t\hat{A}_t\leq\widetilde{C}\log^{\frac{3}{2}}\frac{3T}{\delta}$. Therefore, there exists a $\tilde{t}\in\nbb$ satisfying $\widetilde{T}\leq\tilde{t}\leq T$ and
\begin{equation}\label{last-1}
  \hat{A}_{\tilde{t}}\leq\big[\sum_{t=\widetilde{T}}^{T}\eta_t\big]^{-1}\widetilde{C}\log^{\frac{3}{2}}\frac{3T}{\delta}.
\end{equation}

Taking expectations only with respect to $z$ over both sides of \eqref{grad-1} gives
\begin{equation}\label{last-m1}
 \hat{A}_{t+1}\leq \hat{A}_t-\eta_t\big\langle\phi'(y_t,f_t(x_t))K_{x_t},\nabla\ecal(f_t)\big\rangle+\frac{L\kappa^{2(1+\alpha)}\eta_t^{1+\alpha}}{1+\alpha}|\phi'(y_t,f_t(x_t))|^{1+\alpha}.
\end{equation}
According to \eqref{grad-2}, the term $|\phi'(y_t,f_t(x_t))|^{1+\alpha}$ can be controlled by
\begin{align*}
   |\phi'(y_t,f_t(x_t))|^{1+\alpha} &\leq 2^\alpha|\phi'(y_t,f_t(x_t))-\phi'(y_t,f_H(x_t))|^{1+\alpha}+2^\alpha|\phi'(y_t,f_H(x_t))|^{1+\alpha}\\
   &\leq 2^\alpha L^{1+\alpha}|\inn{f_t-f_H,K_{x_t}}|^{\alpha(1+\alpha)}+2^\alpha|\phi'(y_t,f_H(x_t))|^{1+\alpha}\\
   &\leq 2^\alpha L^{1+\alpha}\kappa^{\alpha(1+\alpha)}\|f_t-f_H\|^{\alpha(1+\alpha)}+2^\alpha|\phi'(y_t,f_H(x_t))|^{1+\alpha}.
\end{align*}
Plugging the above bound into \eqref{last-m1} gives
\begin{multline}\label{last-one-step-progress}
   \hat{A}_{t+1}\leq \hat{A}_t-\eta_t\|\nabla\ecal(f_t)\|^2+\eta_t\big\langle\nabla\ecal(f_t)-\phi'(y_t,f_t(x_t))K_{x_t},\nabla\ecal(f_t)\big\rangle\\
   +\Big(\tilde{a}\|f_t-f_H\|^{\alpha(1+\alpha)}+\tilde{b}\Big)\eta_t^{1+\alpha},
\end{multline}
where we introduce
$$
  \tilde{a}=2^\alpha L^{2+\alpha}\kappa^{(2+\alpha)(1+\alpha)}(1+\alpha)^{-1}\quad\text{and}\quad\tilde{b}=2^\alpha L\kappa^{2(1+\alpha)}(1+\alpha)^{-1}\sup_{z\in\zcal}|\phi'(y,f_H(x))|^{1+\alpha}.
$$

Taking a summation from $t=\tilde{t}$ to $T$ yields
\begin{equation}\label{last-2}
  \hat{A}_{T+1}\leq \hat{A}_{\tilde{t}}+\sum_{t=\tilde{t}}^{T}\Big(\tilde{a}\|f_t-f_H\|^{\alpha(1+\alpha)}+\tilde{b}\Big)\eta_t^{1+\alpha}-\sum_{t=\tilde{t}}^{T}\eta_t\|\nabla\ecal(f_t)\|^2+\sum_{t=\tilde{t}}^{T}\bar{\xi}_t,
\end{equation}
where we introduce the following two sequences of functionals
\begin{gather*}
  \bar{\xi}_t=\eta_t\big\langle\nabla\ecal(f_t)-\phi'(y_t,f_t(x_t))K_{x_t},\nabla\ecal(f_t)\big\rangle,\\
  \bar{\xi}_t'=\eta_t\big\langle\nabla\ecal(f_t)-\phi'(y_t,f_t(x_t))K_{x_t},\nabla\ecal(f_t)\big\rangle\mathbb{I}_{\{\|f_t-f_H\|^2\leq \bar{C}\log\frac{3T}{\delta}\}}.
\end{gather*}
Under the event $\Omega$, it is clear $\bar{\xi}_t=\bar{\xi}'_t$. In the following, we will use Part (b) of Lemma \ref{lem:martingale} to estimate $\sum_{t=\tilde{t}}^{T}\bar{\xi}'_t$.
It is clear that $\ebb_{z_t}[\bar{\xi}'_t]=0$ for all $t\in\nbb$.
Let $\bar{t}$ be any integer in $[\widetilde{T},T]$.
It follows from Lemma \ref{lem:grad-bound} with $\beta=1$ and the definition of $C_4$ given in \eqref{C-3} that (note $\ebb[(\xi-\ebb[\xi])^2<\ebb[\xi^2]$ for a real-valued random variable $\xi$)
\begin{align}
  \sum_{t=\bar{t}}^{T}\ebb_{z_t}\big(\bar{\xi}'_t&-\ebb_{z_t}[\bar{\xi}'_t]\big)^2 \leq \sum_{t=\bar{t}}^{T}\eta_t^2\ebb_{z_t}\big[\big\langle\phi'(y_t,f_t(x_t))K_{x_t},\nabla\ecal(f_t)\big\rangle^2\big]\mathbb{I}_{\{\|f_t-f_H\|^2\leq \bar{C}\log\frac{3T}{\delta}\}}\notag \\
   & \leq \sum_{t=\bar{t}}^{T}\eta_t^2\kappa^2\|\nabla\ecal(f_t)\|^2\ebb_{z_t}\big[|\phi'(y_t,f_t(x_t))|^2\big]\mathbb{I}_{\{\|f_t-f_H\|^2\leq \bar{C}\log\frac{3T}{\delta}\}}\notag\\
   & \leq\sum_{t=\bar{t}}^{T}\eta_t^2\|\nabla\ecal(f_t)\|^2\big(4\kappa^2L^{\frac{1}{\alpha}}\big[\ecal(f_t)-\ecal(f_H)\big]+C_4\big)\mathbb{I}_{\{\|f_t-f_H\|^2\leq \bar{C}\log\frac{3T}{\delta}\}}.\label{last-var}
\end{align}
Analyzing analogously to \eqref{nece-3}, one can show that $\nabla\ecal$ is $(\alpha,L\kappa^{1+\alpha})$-H\"older continuous. Then, Lemma \ref{lem:smooth-hilbert} together with $\nabla\ecal(f_H)=0$ shows that
\begin{equation}\label{last-3}
  \hat{A}_t=\ecal(f_t)-\ecal(f_H)\leq \frac{L\kappa^{1+\alpha}\|f_t-f_H\|^{1+\alpha}}{1+\alpha}.
\end{equation}
Plugging the above inequality into \eqref{last-var} shows
\begin{equation}\label{last-variance}
  \sum_{t=\bar{t}}^{T}\ebb_{z_t}\big(\bar{\xi}'_t-\ebb_{z_t}[\bar{\xi}'_t]\big)^2\leq C_8\log\frac{3T}{\delta}\sum_{t=\bar{t}}^{T}\eta_t^2\|\nabla\ecal(f_t)\|^2
  \leq \eta_{\widetilde{T}}C_8\log\frac{3T}{\delta}\sum_{t=\bar{t}}^{T}\eta_t\|\nabla\ecal(f_t)\|^2,
\end{equation}
where we have used $\bar{t}\geq\widetilde{T}$ and introduced
$$
  C_8=\frac{4\kappa^{3+\alpha}L^{1+\frac{1}{\alpha}}\bar{C}}{1+\alpha}+C_4.
$$
According to \eqref{iterate-prob-3}, there holds
\begin{align*}
  \bar{\xi}_t'-\ebb_{z_t}[\bar{\xi}_t'] & \leq \eta_t\big|\big\langle\phi'(y_t,f_t(x_t))K_{x_t}-\nabla\ecal(f_t),\nabla\ecal(f_t)\big\rangle\big|\mathbb{I}_{\{\|f_t-f_H\|^2\leq \bar{C}\log\frac{3T}{\delta}\}} \\
   & \leq \eta_t\|\nabla\ecal(f_t)\|\big\|\phi'(y_t,f_t(x_t))K_{x_t}-\ebb_{z_t}[\phi'(y_t,f_t(x_t))K_{x_t}]\big\|\mathbb{I}_{\{\|f_t-f_H\|^2\leq \bar{C}\log\frac{3T}{\delta}\}}\\
   & \leq \big(C_3+2L\|f_t-f_H\|^\alpha\kappa^{\alpha+1}\big)\eta_t\|\nabla\ecal(f_t)\|\mathbb{I}_{\{\|f_t-f_H\|^2\leq \bar{C}\log\frac{3T}{\delta}\}},\quad\forall t\geq\bar{t}.
\end{align*}
Due to the $(\alpha,L\kappa^{1+\alpha})$-H\"older continuity of $\nabla\ecal$
$$
  \|\nabla\ecal(f_t)\|=\|\nabla\ecal(f_t)-\nabla\ecal(f_H)\|\leq L\kappa^{1+\alpha}\|f_t-f_H\|^\alpha,
$$
we further get
\begin{align*}
  \bar{\xi}_t'-\ebb_{z_t}[\bar{\xi}_t'] &\leq \eta_t\big(C_3+2L\kappa^{\alpha+1}\big)\max\big(\|f_t-f_H\|^\alpha,1\big)L\kappa^{1+\alpha}\|f_t-f_H\|^\alpha\mathbb{I}_{\{\|f_t-f_H\|^2\leq \bar{C}\log\frac{3T}{\delta}\}}\\
  & \leq \eta_{\widetilde{T}}\big(C_3+2L\kappa^{\alpha+1}\big)L\kappa^{1+\alpha}\bar{C}\log\frac{3T}{\delta}:=\eta_{\widetilde{T}}C_9\log\frac{3T}{\delta},\quad\forall t\geq\bar{t}.
\end{align*}
We can find a $\rho_2>0$ independent of $T$ such that $(e^{\rho_2}-\rho_2-1)C_8\leq\rho_2 C_9$.
Applying Part (b) of Lemma \ref{lem:martingale} with the above bounds on variances and magnitudes of $\bar{\xi}'_k$ followed with union bounds on probabilities, we can find a subset $\Omega'=\{(z_1,\ldots,z_T):z_1,\ldots,z_T\in\zcal\}\subset\zcal^T$ with $\text{Pr}\{\Omega'\}\geq1-\delta$ such that for any $(z_1,\ldots,z_T)\in\Omega'$ there holds (note $\ebb_{z_t}[\bar{\xi}'_t]=0$)
\begin{align*}
  \sum_{t=\bar{t}}^{T}\bar{\xi}_t' &\leq \frac{\eta_{\widetilde{T}}(e^{\rho_2}-\rho_2-1)C_8\log\frac{3T}{\delta}\sum_{t=\bar{t}}^{T}\eta_t\|\nabla\ecal(f_t)\|^2}{\eta_{\widetilde{T}}\rho_2 C_9\log\frac{3T}{\delta}}+\frac{\eta_{\widetilde{T}}C_9\log^2\frac{3T}{\delta}}{\rho_2}\\
  & \leq \sum_{t=\bar{t}}^{T}\eta_t\|\nabla\ecal(f_t)\|^2+\frac{\eta_{\widetilde{T}}C_9\log^2\frac{3T}{\delta}}{\rho_2},\quad\forall \bar{t}\in[\widetilde{T},T].
\end{align*}
Under the event $\Omega\cap\Omega'$, we can plug the above inequality with $\bar{t}=\tilde{t},\bar{\xi}'_t=\bar{\xi}_t$ and $\|f_t-f_H\|^2\leq\bar{C}\log\frac{3T}{\delta},\forall t=1,\ldots,T$ into \eqref{last-2} to derive
\begin{align*}
  \hat{A}_{T+1} &\leq \hat{A}_{\tilde{t}}+\Big(\tilde{a}\bar{C}\log\frac{3T}{\delta}+\tilde{b}\Big)\sum_{t=\tilde{t}}^{T}\eta_t^{1+\alpha}
  +\frac{\eta_{\widetilde{T}}C_9\log^2\frac{3T}{\delta}}{\rho_2}\\
  & \leq \big[\sum_{t=\widetilde{T}}^{T}\eta_t\big]^{-1}\widetilde{C}\log^{\frac{3}{2}}\frac{3T}{\delta}+\Big(\tilde{a}\bar{C}+\tilde{b}\Big)\log\frac{3T}{\delta}\sum_{t=\tilde{t}}^{T}\eta_t^{1+\alpha}+
  \frac{\eta_{\widetilde{T}}C_9\log^2\frac{3T}{\delta}}{\rho_2},
\end{align*}
where the last inequality is due to \eqref{last-1}. This establishes the stated inequality with probability $1-\delta$ and
$$
  \widetilde{C}'=\widetilde{C}+\tilde{a}\bar{C}+\tilde{b}+C_9\rho_2^{-1}.
$$
It is clear that $\widetilde{C}'$ is independent of $T$ and $\widetilde{T}$. The proof is complete.
\end{proof}

\begin{proof}[Proof of Corollary \ref{cor:rate-last}]
  The polynomially decaying step sizes $\eta_t=\eta_1t^{-\theta} (\theta>\frac{1}{2})$ satisfies the monotonicity and $\sum_{t=1}^{\infty}\eta_t^2<\infty$
  Furthermore, we have
  $$
    \big[\sum_{t=\half}^{T}\eta_t\big]^{-1}\leq \frac{2}{T\eta_T}=O(T^{\theta-1})
    \quad\text{and}\quad
    \sum_{t=\half}^{T}\eta_t^{1+\alpha}\leq \frac{(T+1)\eta_{\half}^{1+\alpha}}{2}=O(T^{1-(1+\alpha)\theta}).
  $$
  The proof is complete if we plug the above estimates into Theorem \ref{thm:last}.
\end{proof}

\section*{Acknowledgements}
The work described in this paper is supported partially by the National Natural Science Foundation of China (Grants No.11401524, 11531013, 11571078, 11631015). Lei Shi is also supported by the Joint Research Fund by National Natural Science Foundation of China and Research Grants Council of Hong Kong (Project No. 11461161006 and Project No. CityU 104012) and Zhuo Xue program of Fudan University. The corresponding author is Zheng-Chu Guo.

\appendix
\renewcommand{\thesection}{{\Alph{section}}}
\renewcommand{\thesubsection}{\Alph{section}.\arabic{subsection}}
\renewcommand{\thesubsubsection}{\Roman{section}.\arabic{subsection}.\arabic{subsubsection}}
\setcounter{secnumdepth}{-1}
\setcounter{secnumdepth}{3}
\section{Some Additional Proofs}\label{sec:additional-proofs}
\begin{proof}[Proof of Lemma \ref{lem:series}]
  Let $\epsilon>0$ be an arbitrary number. Since $\lim_{t\to\infty}\eta_t=0$ we can find a $t_3\in\nbb$ such that $\eta_t\leq \frac{\epsilon}{2}$ for all $t\geq t_3$.
  Since $\sum_{t=1}^{\infty}\eta_t=\infty$, we can also find a $t_4>t_3$ such that $\sum_{k=1}^{t_3}\eta_k^2\leq \frac{\epsilon}{2}\sum_{k=1}^{t_4}\eta_k$. Then, for any $t\geq t_4$, it holds
  \begin{align*}
    \big[\sum_{k=1}^{t}\eta_k\big]^{-1}\sum_{k=1}^{t}\eta_k^2 &= \big[\sum_{k=1}^{t}\eta_k\big]^{-1}\sum_{k=1}^{t_3}\eta_k^2+\big[\sum_{k=1}^{t}\eta_k\big]^{-1}\sum_{k=t_3+1}^{t}\eta_k^2 \\
    &\leq \frac{\epsilon}{2}+\frac{\epsilon}{2}\big[\sum_{k=1}^{t}\eta_k\big]^{-1}\sum_{k=t_3+1}^{t}\eta_k\leq\epsilon.
  \end{align*}
  Since $\epsilon>0$ is arbitrarily chosen, the proof is complete.
\end{proof}
\begin{proof}[Proof of Lemma \ref{lem:smooth-hilbert}]
  Fix $f,\tilde{f}\in H$. Define a function $g:\rbb\to\rbb$ by $g(t)=\gcal(\tilde{f}+t(f-\tilde{f}))$. It is clear that
  $
    g'(t)=\inn{f-\tilde{f},\nabla\gcal(\tilde{f}+t(f-\tilde{f}))}
  $
  and
  \begin{align*}
    |g'(t)-g'(\tilde{t})| & = \Big\langle f-\tilde{f},\nabla\gcal(\tilde{f}+t(f-\tilde{f}))-\nabla\gcal(\tilde{f}+\tilde{t}(f-\tilde{f}))\Big\rangle \\
     & \leq \|f-\tilde{f}\|\big\|\nabla\gcal\big(\tilde{f}+t(f-\tilde{f})\big)-\nabla\gcal\big(\tilde{f}+\tilde{t}(f-\tilde{f})\big)\big\| \\
     & \leq L\|f-\tilde{f}\|^{1+\alpha}|t-\tilde{t}|^{\alpha}.
  \end{align*}
  It then follows that
  \begin{align*}
    g(1)-g(0)-g'(0) & =\int_0^1[g'(t)-g'(0)]dt\leq \int_0^1|g'(t)-g'(0)|dt \\
     & \leq L\|f-\tilde{f}\|^{1+\alpha}\int_0^1t^\alpha dt=\frac{L\|f-\tilde{f}\|^{1+\alpha}}{1+\alpha},
  \end{align*}
  which amounts to the second inequality in \eqref{smooth-hilbert}
  \begin{equation}\label{smooth-hilbert-1}
    \gcal(f)\leq \gcal(\tilde{f})+\inn{f-\tilde{f},\nabla\gcal(\tilde{f})}+\frac{L\|f-\tilde{f}\|^{1+\alpha}}{1+\alpha}.
  \end{equation}

  We now turn to the first inequality in \eqref{smooth-hilbert}.
  Fix $f$ and $\tilde{f}\in H$.
  Define a functional $\lcal:H\to\rbb$ by $\lcal(\bar{f})=\gcal(\bar{f})-\inn{\bar{f},\nabla\gcal(f)}$. It is clear that $\lcal$ is a convex function and $\nabla\lcal(f)=\nabla\gcal(f)-\nabla\gcal(f)=0$. According to the first-order optimality condition, we know $\lcal$ attains its minimum at $f$ and
  \begin{align*}
    \lcal(f) & = \min_{\bar{f}\in H}\lcal(\bar{f}) = \min_{\bar{f}\in H}\big[\gcal(\bar{f})-\inn{\bar{f},\nabla \gcal(f)}\big]\\
     & \leq \min_{\bar{f}\in H}\Big[\gcal(\tilde{f})+\inn{\bar{f}-\tilde{f},\nabla \gcal(\tilde{f})}+\frac{L\|\tilde{f}-\bar{f}\|^{1+\alpha}}{1+\alpha}-\inn{\bar{f},\nabla\gcal(f)}\Big] \\
     & = \lcal(\tilde{f})+\min_{\bar{f}\in H}\Big[\inn{\tilde{f}-\bar{f},\nabla\gcal(f)-\nabla\gcal(\tilde{f})}+\frac{L\|\tilde{f}-\bar{f}\|^{1+\alpha}}{1+\alpha}\Big]\\
     & = \lcal(\tilde{f})+\min_{\bar{f}\in H}\Big[\inn{\bar{f},\nabla\gcal(f)-\nabla\gcal(\tilde{f})}+\frac{L\|\bar{f}\|^{1+\alpha}}{1+\alpha}\Big],
  \end{align*}
  where the inequality follows from \eqref{smooth-hilbert-1}.
  Taking $\bar{f}=L^{-\frac{1}{\alpha}}\|\nabla\gcal(\tilde{f})-\nabla\gcal(f)\|^{\frac{1-\alpha}{\alpha}}\big(\nabla\gcal(\tilde{f})-\nabla\gcal(f)\big)$ in the above inequality, we derive
  \begin{align*}
    \lcal(f)&\leq \lcal(\tilde{f})-L^{-\frac{1}{\alpha}}\|\nabla\gcal(\tilde{f})-\nabla\gcal(f)\|^{\frac{1+\alpha}{\alpha}}+\frac{L^{-\frac{1}{\alpha}}\|\nabla\gcal(\tilde{f})-\nabla\gcal(f)\|^{\frac{1+\alpha}{\alpha}}}{1+\alpha}\\
    &= \lcal(\tilde{f})-\frac{\alpha L^{-\frac{1}{\alpha}}}{1+\alpha}\|\nabla\gcal(f)-\nabla\gcal(\tilde{f})\|^{\frac{1+\alpha}{\alpha}}.
  \end{align*}
  This establishes the first inequality in \eqref{smooth-hilbert}. The proof is complete.
\end{proof}


\begin{proof}[Proof of Proposition \ref{prop:unif-average}]
Consider the following sequence of functionals $\tilde{\xi}_k,k=1,\ldots,T$ by
$$
  \tilde{\xi}_k=\big\langle f_H-f_k,\phi'(y_k,f_k(x_k))K_{x_k}-\ebb_{z_k}\big[\phi'(y_k,f_k(x_k))K_{x_k}\big]\big\rangle,
$$
where $\bar{C}$ is defined in Proposition \ref{prop:iterate-bound-pr}. Eq. \eqref{conv-prob-1} implies that
$$
  |\tilde{\xi}_k|\mathbb{I}_{\{\|f_k-f_H\|^2\leq \bar{C}\log\frac{2T}{\delta}\}}\leq \big(C_3+2L\kappa^{\alpha+1}\big)\bar{C}\log\frac{2T}{\delta}.
$$
By Part (a) of Lemma \ref{lem:martingale}, there exists a subset $\Omega'=\{(z_1,\ldots,z_T):z_1,\ldots,z_T\in\zcal\}\subset\zcal^T$ with probability measure $\text{Pr}\{\Omega'\}\geq 1-\frac{\delta}{2}$ such that for any $(z_1,\ldots,z_T)\in\Omega'$ the following inequality holds
\begin{align*}
  \sum_{k=1}^{T}\tilde{\xi}_k\mathbb{I}_{\{\|f_k-f_H\|^2\leq \bar{C}\log\frac{2T}{\delta}\}}
  \leq \big(C_3+2L\kappa^{\alpha+1}\big)\bar{C}\log\frac{2T}{\delta}\Big(2T\log\frac{2}{\delta}\Big)^{\frac{1}{2}}.
\end{align*}
According to Proposition \ref{prop:iterate-bound-pr}, there exists a subset $\Omega=\{(z_1,\ldots,z_T):z_1,\ldots,z_T\in\zcal\}\subset\zcal^T$ with probability measure $\text{Pr}\{\Omega\}\geq1-\frac{\delta}{2}$ such that for any $(z_1,\ldots,z_T)\in\Omega$ there holds the inequality $\max_{1\leq k\leq T}\|f_k-f_H\|^2\leq\bar{C}\log\frac{2T}{\delta}$. Under the event $\Omega\cap\Omega'$, we then have
\begin{equation}\label{uniform-average-1}
  \sum_{k=1}^{T}\tilde{\xi}_k\leq \big(C_3+2L\kappa^{\alpha+1}\big)\bar{C}\log\frac{2T}{\delta}\Big(2T\log\frac{2}{\delta}\Big)^{\frac{1}{2}}.
\end{equation}

Furthermore, it follows from \eqref{iterate-prob-0} that
$$
  2[\ecal(f_k)-\ecal(f_H)]\leq \eta_k^{-1}\big[\|f_k-f_H\|^2-\|f_{k+1}-f_H\|^2\big]+\eta_k\kappa^2\big(A\phi(y_k,f_k(x_k))+B\big) + 2\tilde{\xi}_k.
$$
Taking a summation of the above inequality from $k=1$ to $T$ yields the following inequality under the event $\Omega\cap\Omega'$
\begin{multline}\label{uniform-average-2}
  2\sum_{k=1}^{T}[\ecal(f_k)-\ecal(f_H)]
  \leq \sum_{k=1}^{T-1}\big(\eta_{k+1}^{-1}-\eta_k^{-1}\big)\|f_{k+1}-f_H\|^2 + \eta_1^{-1}\|f_1-f_H\|^2 \\+ \kappa^2\sum_{k=1}^{T}\eta_k\big(A\phi(y_k,f_k(x_k))+B\big)+2\sum_{k=1}^{T}\tilde{\xi}_k.
\end{multline}
It follows from \eqref{iterate-bound-ae-1} that
$$
  \sum_{k=1}^{T}\eta_k\phi(y_k,f_k(x_k))\leq\|f_H\|^2+2\sum_{k=1}^{T}\eta_k\phi(y_k,f_H(x_k))+\kappa^2B\sum_{k=1}^{T}\eta_k^2.
$$
Plugging the above bound into \eqref{uniform-average-2} and using the monotonicity of $\eta_k$ together with \eqref{uniform-average-1}, we derive the following inequality with probability at least $1-\delta$
\begin{multline*}
  2\sum_{k=1}^{T}[\ecal(f_k)-\ecal(f_H)] \leq
  (A\kappa^2+\eta_1^{-1})\|f_H\|^2 +\kappa^2\sum_{k=1}^{T}\Big(2A\eta_k\sup_{z}\phi(y,f_H(x))+B\eta_k+AB\kappa^2\eta_k^2\Big)\\+\big(C_3+2L\kappa^{\alpha+1}\big)\bar{C}(8T)^{\frac{1}{2}}\log^{\frac{3}{2}}\frac{2T}{\delta}.
\end{multline*}
The proof is complete.
\end{proof}

\vskip 0.2in

\end{document}